\documentclass{article}

\usepackage[accepted]{icml2019}

\usepackage[utf8]{inputenc} 
\usepackage[T1]{fontenc}    

\usepackage{hyperref}       
\usepackage{url}            
\usepackage{booktabs}       
\usepackage{amsfonts}       
\usepackage{nicefrac}       
\usepackage{microtype}      
\usepackage{graphicx} 
\usepackage{amssymb}
\usepackage{amsmath}
\usepackage{bbm}
\usepackage{amsfonts}
\usepackage{amsthm}
\usepackage{graphicx}
\usepackage{mathtools}
\usepackage{thmtools, thm-restate}
\usepackage{subcaption}
\usepackage{caption}
\usepackage{supertabular}
\usepackage{datetime}
\usepackage{xcolor}
\usepackage{diagbox}
\usepackage{longtable}
\usepackage{algorithm}
\usepackage{algorithmic}
\usepackage{natbib}
\usepackage{lipsum}

\usepackage{dsfont}

\graphicspath{ {images/} }

\newtheorem{lemma}{Lemma}

\newtheorem{proposition}{Proposition}

\allowdisplaybreaks

\icmltitlerunning{Robust Learning from Untrusted Sources}

\usepackage{xspace}
\makeatletter
\DeclareRobustCommand\onedot{\futurelet\@let@token\@onedot}
\def\@onedot{\ifx\@let@token.\else.\null\fi\xspace}

\def\iid{{i.i.d}\onedot}
\def\eg{{e.g}\onedot} 
\def\ie{{i.e}\onedot}

\makeatother

\begin{document}
\setlength{\dbltextfloatsep}{8pt plus 1.0pt minus 2.0pt}
\setlength{\dblfloatsep}{8pt plus 1.0pt minus 2.0pt}
\setlength{\intextsep}{8pt plus 1.0pt minus 2.0pt}
\addtolength{\abovedisplayskip}{-.15\baselineskip}
\addtolength{\belowdisplayskip}{-.15\baselineskip}
\addtolength{\abovedisplayshortskip}{-1.0\baselineskip}
\addtolength{\belowdisplayshortskip}{-1.0\baselineskip}

\twocolumn[
\icmltitle{Robust Learning from Untrusted Sources}



\icmlsetsymbol{equal}{*}

\begin{icmlauthorlist}
\icmlauthor{Nikola Konstantinov}{IST}
\icmlauthor{Christoph H. Lampert}{IST}
\end{icmlauthorlist}

\icmlaffiliation{IST}{Institute of Science and Technology, Klosterneuburg, Austria}

\icmlcorrespondingauthor{Nikola Konstantinov}{nkonstan@ist.ac.at}

\icmlkeywords{Robustness, domain adaptation, untrusted, sources, PAC learning}

\vskip 0.3in
]

\printAffiliationsAndNotice{}  

\begin{abstract}
Modern machine learning methods often require more data for training than a single expert can provide. Therefore, it has become a standard procedure to collect data from multiple external sources, \eg via crowdsourcing. Unfortunately, the quality of these sources is not always guaranteed. As further complications, the data might be stored in a distributed way, or might even have to remain private. In this work, we address the question of how to learn robustly in such scenarios. Studying the problem through the lens of statistical learning theory, we derive a procedure that allows for learning from all available sources, yet automatically suppresses irrelevant or corrupted data. We show by extensive experiments that our method provides significant improvements over alternative approaches from robust statistics and distributed optimization.
\end{abstract}

\section{Introduction}
\label{sec:intro}
Due to the outstanding performance of modern machine learning algorithms on various real-world tasks, there is an increasing amount of interest by practitioners in producing predictive models, specific to their purposes. In many application domains, however, it may be prohibitively expensive for a single expert to produce a high-quality labeled dataset, that is large enough for training a good model. Therefore, it has become a common practice to obtain data from various external data sources. Examples range from the use of crowdsourcing platforms, through collecting data from different websites and social networks profiles, to collaborating with other parties working in similar domains.

Naturally, datasets obtained from such sources vary greatly in quality, reliability and relevance for the learning task. For instance, genetic data from multiple laboratories may have been obtained via different measurement devices or data preprocessing techniques \cite{wahlsten2003different}. In the case of crowdsourcing, a typical problem is label bias and label noise, due to incompetent or malicious workers \cite{Wais10towardsbuilding}. More generally, statistical and machine learning models are known to suffer in performance due to gross errors, contaminations and adversarial modifications of the data \cite{tukey1960survey, Biggio}. The variety of possible deviations from the target data distribution, as well as the large volume and dimensionality of the data in real-world applications, make the assessment of the quality of the provided data a difficult task. An additional complication is that the data might have to remain decentralized, because of high communication costs, or it might not be directly available for inspection, due to privacy constraints. 

In this paper we study the problem of \textit{how to learn from multiple untrusted sources}, while being \textit{robust to any corruptions} of the data provided from each of them. As an alternative to the naive approaches of simply training on all data or only on a trusted subset, we propose a method that automatically assigns weights to the sources. To this end, we build up on techniques from the domain adaptation literature and prove an upper bound on the expected loss of a predictor, learned by minimizing any weighted version of the empirical loss. Based on these theoretical insights, our algorithm selects the weights for the sources by approximately minimizing this upper bound.

Intuitively, \textit{the weights are assigned to the sources according to the quality and reliability of the data they provide}, quantified by an appropriate measure of trust we introduce. This is achieved by comparing the data from each source to a \textit{small reference dataset}, obtained or trusted by the learner. The measure can also be computed locally at every source or by a gradient-based optimization procedure, which allows for the implementation of the algorithm under \textit{privacy constraints}, as well as its integration into any \textit{standard distributed learning framework}.

We perform an extensive experimental evaluation \footnote{Code is available at \href{https://github.com/NikolaKon1994/Robust-Learning-from-Untrusted-Sources}{https://github.com/NikolaKon1994/Robust-Learning-from-Untrusted-Sources}} of our algorithm and demonstrate its ability to learn from all available data, while successfully suppressing the effect of corrupted or irrelevant sources. It consistently outperforms both naive approaches of learning on all available data directly or learning on the reference dataset only, \textit{for any amount and any type of data contamination} considered. We also observe its performance to be superior to multiple baseline methods from robust statistics and robust distributed learning.

\section{Related work}
\label{sec:related_work}
Learning from multiple sources is a topic relevant for many applications of machine learning and data corruption is a problem acknowledged in some of these areas. In particular, \cite{bi2014learning, kajino2012convex, awasthi2017efficient} and references therein consider the problem of label noise in \textit{crowdsourced data}. The non-\iid split of the data on local devices is one of the main characteristics of \textit{federated learning} and the pioneering work of \cite{mcmahan2017communication} addresses this by occasionally averaging local models to ensure global consistency. Fault tolerance and prevention of sybil attacks in federated learning have been considered by \cite{smith2017federated} and \cite{fung2018mitigating} respectively. Robustness has also been explored in the context of \textit{multi-view} learning, where data arrives from various feature extractors \cite{zhao2017multi, xie2017robust, zhang2017robust}.

The work closest in spirit to ours is the one of \cite{qiao2018learning}, who provide efficient algorithms for learning from batches of data, an $\epsilon$-fraction of which can be malicious. Their focus is different though, as they only study algorithms for learning discrete distributions and explore the regime where each data source provides a small amount of samples. In contrast, we are interested in general supervised learning problems and work in a setting where more data is available per source. \cite{charikar2017learning, NIPS2018_8246} also study learning with a reference dataset as a protection against data corruption, but focus on a single untrusted dataset only and on convex objectives and label noise respectively. There is a vast body of literature focusing on robustness of learning algorithms to corruptions \textit{within} a dataset, \eg \cite{tukey1960survey, huber2011robust, diakonikolas2016robust, prasad2018robust}, and on identifying data corruptions at \textit{prediction time}, \eg \cite{hendrycks2017baseline,sun2018ksconf}. These lines of work are orthogonal to ours, since we consider \textit{multiple training datasets}, some of which are corrupted, and hence a literature review in this direction is beyond the scope of this paper.

Another related area is the one of \textit{robust distributed learning and optimization}. The work of \cite{feng2014distributed} develops a method for implementing any robust supervised learning algorithm in a distributed manner. \cite{feng2017fundamental} provide lower bounds for the communication complexity of PAC learning in a distributed environment, in the presence of malicious outliers. Fault tolerance and resistance to adversarial behavior of individual nodes in a distributed system have been studied from the point of view of Byzantine-robust distributed optimization, \eg \cite{blanchard2017machine, pmlr-v80-yin18a, NIPS2018_7712}. These works consider arbitrary (even adversarial) behavior of the nodes, however they study the convergence of gradient-based optimization procedures and typically have to assume that at least half of the nodes behave normally. In contrast, we are interested in the generalization performance of empirical risk minimizers and make no assumptions about the number of corrupted sources. The worst-case performance of distributed SGD has also been studied in the context of asynchronous training \cite{NIPS2015_5717, Alistarh:2018:CSG:3212734.3212763}.

On the methodological level, we borrow techniques from the field of domain adaptation. To measure the difference between data distributions, we use the same integral probability metric as \cite{mohri2012new, zimin2017learning}. The problem we study is related to \textit{multi-source domain adaptation}, \eg \cite{crammer2008learning, ben2010theory}, and to \textit{multi-task learning} with unlabeled data \cite{PenLam17}. In particular, our Theorem \ref{thm:main_bound} is similar to a result in \cite{zhang2012generalizationArxiv}. We refer to the paragraph after Theorem \ref{thm:main_bound} for a more detailed comparison. However, all these works focus on sharing information between similar domains, in order to obtain better predictors for a target task, while we are interested in applying such techniques for detecting untrustworthy sources of data and improving the robustness of the learning procedure. 

A relation between robustness and domain adaptation has been explored in the work of \cite{mansour2014robust}, who use a property called \textit{algorithmic robustness} to derive generalization bounds for domain adaptation. Another related line of work is the one of \cite{mansour2009domain, NIPS2018_8046}, who provide guarantees for a classifier learned on data from $N$ domains on any target distribution that is a mixture of the distributions of the sources. Domain adaptation techniques were also used by \cite{song2018improving}, for improving the test-time robustness of predictive models to adversarial examples.

\section{Robust learning from untrusted sources}
\label{sec:main_results_section}
Given a \textit{small reference dataset}, we want to leverage additional training data from \textit{multiple untrusted sources} in an optimal way, so that the obtained predictor performs well on a target distribution. A naive approach will be to trust all data, merge it into one dataset and train end-to-end to obtain a predictive model. Such an approach will intuitively be vulnerable to irrelevant or low-quality data provided by some sources. In this section, we design a more \textit{robust} algorithm that instead minimizes a weighted empirical loss.
\subsection{Theory}
\label{sec:theory}
\textbf{Setup.} Let $\mathcal{X}$ be an input space and $\mathcal{Y}$ be an output space. Our theoretical setup covers both the case of classification ($\mathcal{Y} = \{1, 2, \ldots, K\}$) and regression ($\mathcal{Y} = \mathbb{R}$). We assume that the learner has access to a \textit{small reference dataset} $S_T \vcentcolon= \{\left(x_{T, 1}, y_{T,1}\right), \ldots, \left(x_{T, m_T}, y_{T, m_T}\right)\}$ of $m_T$ samples drawn \iid from a target distribution $\mathcal{D}_T$ over $\mathcal{X}\times \mathcal{Y}$. In addition, training data is available from $N$ \textit{untrusted data sources}, each of them characterized by its own distribution, $\mathcal{D}_i$, over $\mathcal{X}\times\mathcal{Y}$, possibly different from $\mathcal{D}_T$. We denote the number of samples from source $i$ by $m_i$. Let the corresponding \iid datasets be $S_i \vcentcolon= \{\left(x_{i,1}, y_{i,1}\right), \ldots, \left(x_{i,m_i}, y_{i, m_i}\right)\} \overset{\iid}{\sim} \mathcal{D}_i$ for each $i = 1, \ldots, N$.

Let $L: \mathcal{Y} \times \mathcal{Y} \rightarrow \mathbb{R}_{+}$ be a loss function, bounded by some $M > 0$. For any distribution $\mathbb{P}$ on $\mathcal{X}\times\mathcal{Y}$ and any function $h: \mathcal{X} \rightarrow \mathcal{Y}$, denote by
$$\epsilon_\mathbb{P} \left(h\right) = \mathbb{E}_{\left(x,y\right)\sim \mathbb{P}}\left(L\left(h(x),y\right)\right)$$
the expected loss of the predictor $h$ with respect to the distribution $\mathbb{P}$. Let $\epsilon_i \left(h\right) = \epsilon_{\mathcal{D}_i} \left(h\right)$ be the expected loss of a predictor $h$ on the distribution of the $i$-th source. Denote by $\hat{\epsilon}_i$ the corresponding empirical counterparts.

Given a hypothesis class $\mathcal{H} \subset \{h:\mathcal{X}\rightarrow\mathcal{Y}\}$, our goal is to use all samples from the \textit{sources} to construct a hypothesis with low expected loss on the target distribution $\mathcal{D}_T$. Note that if we also want to use the reference data at training time, we can simply include it as one of the data sources.

\textbf{Source-specific weights.} For a vector of weights $\alpha = \left(\alpha_1, \ldots, \alpha_N\right)$, such that $\sum_{i=1}^N \alpha_i = 1$ and $\alpha_i \geq 0$ for all $i$, we define the $\alpha$-weighted expected risk of a predictor $h$ as:
\begin{equation}
    \epsilon_{\alpha}(h) =\!\sum_{i=1}^N \alpha_i \epsilon_i (h) = \!\sum_{i=1}^N \alpha_i \mathbb{E}_{(x,y)\sim \mathcal{D}_i}\left(L(h(x), y)\right)
\end{equation}
and its empirical counterpart as:
\begin{equation}
\label{eqn:defn_of_weighted_error}
\begin{split}
\hat{\epsilon}_{\alpha}(h) & =\!\sum_{i=1}^N \alpha_i \hat{\epsilon}_{i}(h) = \!\sum_{i=1}^N \frac{\alpha_i}{m_i}
\!\sum_{j=1}^{m_i}L\left(h(x_{i,j}), y_{i,j}\right).
\end{split}
\end{equation}
With $\mathcal{H}$ as our hypothesis class, let $\hat{h}_{\alpha} = \text{argmin}_{h\in \mathcal{H}} \hat{\epsilon}_{\alpha}\left(h\right)$.

We aim to find weights $\alpha$, such that the predictor $\hat{h}_{\alpha}$ performs well on the target task, \ie such that $\epsilon_T(\hat{h}_{\alpha})$ is small.

\textbf{Evaluating the quality of a source.} Intuitively, a good learning algorithm will assign more weight to sources, whose distribution is similar to the target one, and less weight to those that provide different or low-quality data. Although any standard distance measure on the space of distributions could in theory be used to measure such differences, most of them would not provide any guarantees on the performance of the learned classifier. Furthermore, most similarity measures between distributions, \eg the Kullback-Leibler divergence, are hard to estimate from finite data and overly strict, as they are independent of the learning setup.

We therefore adopt a specific notion of distance that depends on the hypothesis class and allows us to reason about the change in performance of a predictor from $\mathcal{H}$ learned on one distribution, but applied to the other. Following \cite{mohri2012new}, we define the \textit{discrepancy} between the distributions $\mathcal{D}_i$ and $\mathcal{D}_T$ with respect to the hypothesis class $\mathcal{H}$ as:
\begin{align}
\label{eqn:defn_of_disrepancy}
d_{\mathcal{H}}\left(\mathcal{D}_i, \mathcal{D}_T\right) = \sup_{h\in\mathcal{H}}\left(|\epsilon_i (h) - \epsilon_T (h)|\right).
\end{align}
Intuitively, the discrepancy between the two distributions is large, if there exists a predictor that performs well on one of them and badly on the other. On the other hand, if all functions in the hypothesis class perform similarly on both, then $\mathcal{D}_i$ and $\mathcal{D}_T$ have low discrepancy.

The following theorem provides a bound on the expected loss on the target distribution of the predictor $\hat{h}_{\alpha}$, \ie the minimizer of the $\alpha$-weighted sum of the empirical losses over the source data.
\begin{restatable}{theorem}{mainthm}
\label{thm:main_bound}
Given the setup above, let $\hat{h}_{\alpha} = \textrm{argmin}_{h\in\mathcal{H}}\hat{\epsilon}_{\alpha}(h)$ and $h_T^{*} = \textrm{argmin}_{h\in\mathcal{H}}\epsilon_{T}(h)$. For any $\delta > 0$, with probability at least $1 - \delta$ over the data:
\begin{align}
\epsilon_T(\hat{h}_{\alpha}) & \leq \epsilon_T (h_T^{*}) + 4\sum_{i=1}^N \alpha_i \mathcal{R}_i \left(\mathcal{H}\right) + 2\sum_{i=1}^N \alpha_i d_{\mathcal{H}}\left(\mathcal{D}_i, \mathcal{D}_T\right) 
\nonumber
\\ 
 & + 6 \sqrt{\frac{\log\left(\frac{4}{\delta}\right)M^2}{2}}\sqrt{\sum_{i=1}^N\frac{\alpha_i^2}{m_i}},
\label{eqn:final_bound}
\end{align}
where, for each source $i = 1, \ldots, N$,
\begin{equation*}
\mathcal{R}_i \left(\mathcal{H}\right) = \mathbb{E}_{\sigma}\left(\sup_{f\in\mathcal{H}}\left(\frac{1}{m_i}\sum_{j=1}^{m_i}\sigma_{i,j}L(f(x_{i,j}), y_{i,j})\right)\right)
\end{equation*}
\noindent and $\sigma_{i,j}$ are independent Rademacher random variables. 
\end{restatable}
\noindent A proof is provided in the supplementary material. 

We note that a similar result appears as Theorem 5.2 in the arXiv version \cite{zhang2012generalizationArxiv} of the NIPS paper \cite{zhang2012generalizationNIPS}. The authors bound the gap between the weighted empirical loss on the source data of any classifier and its expected loss on the target task, with the additional assumption of a deterministic labeling function for each source. Based on this, they study the asymptotic convergence of domain adaptation algorithms as the sample sizes at all sources go to infinity. In contrast, our theorem compares the performance of the minimizer $\hat{h}_{\alpha}$ of the $\alpha$-weighted empirical loss on the target task to the performance of the optimal (but unknown) $h_T^*$ and does not require deterministic labeling functions. Our target application is also different, since we use the bound to design learning algorithms that are robust to corrupted or irrelevant data, given finite amount of samples from each source.

\subsection{From bound to algorithm}
\label{sec:algorithm}
\textbf{Algorithm description.} To obtain a good predictor for the target task, we would like to choose $\alpha$, such that $\epsilon_T(\hat{h}_{\alpha})$ is as close as possible to $\epsilon_T(h^*_T)$ (the expected loss of the best hypothesis in $\mathcal{H}$). This suggests selecting the weights by minimizing the right-hand side of (\ref{eqn:final_bound}).

While the Rademacher complexities are functions of both the underlying distribution and the hypothesis class, in practice one usually works with a computable upper bound that is distribution-independent (\eg using VC dimension). For some common examples of such bounds we refer to the supplementary material, as well as to \cite{bousquet2004introduction, shalev2014understanding}. In our setting the hypothesis space $\mathcal{H}$ is fixed and therefore these bounds would be identical for all $i$. Therefore, we expect the $\mathcal{R}_i\left(\mathcal{H}\right)$ to be of similar order to each other and the impact of $\alpha$ on the second term in the bound to be negligible. We thus concentrate on optimizing the remaining terms.

Because the true discrepancies are unknown, we estimate them from the data by their empirical counterparts:
\begin{equation}
\label{eqn:defn_of_emp_disrepancy}
\begin{split}
d_{\mathcal{H}}\left(S_i, S_T\right) & = \sup_{h\in\mathcal{H}}\left(|\hat{\epsilon}_i (h) - \hat{\epsilon}_T (h)|\right) \\ & = \sup_{h\in\mathcal{H}}(|\frac{1}{m_i}\sum_{j=1}^{m_i} L\left(h\left(x_{i,j}\right), y_{i,j}\right) \\ & \quad \quad \text{  } - \frac{1}{m_T} \sum_{j=1}^{m_T} L\left(h\left(x_{T, j}\right), y_{T, j}\right)|).
\end{split}
\end{equation} 
In summary, the bound suggests to choose a weighting for the sources by minimizing: 
\begin{equation}
\label{eqn:minimization_formula}
\begin{split}
& \quad \min_{\alpha}  \sum_{i=1}^N \alpha_i d_{\mathcal{H}}\left(S_i, S_T\right) + \lambda \sqrt{\sum_{i=1}^N \frac{\alpha_i^2}{m_i}}, \\ & \text{subject to: } \sum_{i=1}^N \alpha_i = 1 \text{ and } \alpha_i \geq 0 \text{ for all } i,
\end{split}
\end{equation} where $\lambda > 0$ is a hyperparameter that can be selected by cross-validation on the reference dataset. The algorithm then proceeds to minimize the $\alpha$-weighted  empirical risk over the sources (\ref{eqn:defn_of_weighted_error}), possibly with a regularization term. Pseudocode of the algorithm is given in Algorithm \ref{alg:main_algo}.

\begin{algorithm}[t]
   \caption{Robust learning from untrusted sources}
   \label{alg:main_algo}
\begin{algorithmic}
   \STATE {\bfseries Inputs:} 1. Loss $L$, hypothesis set $\mathcal{H}$, parameter $\lambda$
   \STATE \quad \quad \quad \text{ } 2. Reference dataset $S_T$ 
   \STATE \quad \quad \quad \text{ } 3. Datasets $S_1, \ldots, S_N$ from the $N$ sources
   \FOR[Potentially in parallel]{$i=1$ {\bfseries to} $N$}
   \STATE Compute $d_{\mathcal{H}}\left(S_i, S_T\right)$
   \ENDFOR
   \STATE Select $\alpha$ by solving (\ref{eqn:minimization_formula}). 
   \STATE Minimize $\alpha$-weighted loss: $\hat{h}_{\alpha} = \textrm{argmin}_{h\in\mathcal{H}}\hat{\epsilon}_{\alpha}(h)$
   \STATE{\textbf{Return:}} $\hat{h}_{\alpha}$
\end{algorithmic}
\end{algorithm}

\textbf{Discussion.} While derived from our theoretical results, the minimization procedure for selecting the weights also has an intuitive interpretation. Note that the first term in (\ref{eqn:minimization_formula}) is small whenever large weights are paired with small discrepancies and hence encourages trusting sources that provide data similar to the reference target sample. The second term is small whenever the weights are distributed proportionally to the number of samples per source. Thus, it acts as a form of regularization, by encouraging the usage of information from as many sources as possible.

The hyperparamater $\lambda$ controls a trade-off between exploiting similar tasks and leveraging information from all sources. As $\lambda \rightarrow \infty$, all tasks are assigned weights proportional to the amount of training samples they provide and the model minimizes the empirical risk over all the data, regardless of the quality of the samples. In contrast, as $\lambda \rightarrow 0$, the model becomes more sensitive to differences between the source data and the clean reference set, until all weight is assigned to the source closest to the target domain. Assuming that the reference set is included as one of the data sources, these extremes correspond to the naive approach of trusting all sources and training on a merged dataset and not trusting any of them and training on the initial clean data only. In our experiments in Section \ref{sec:experiments} we will see that there is a better operating point between those two extremes. It naturally depends on the actual quality of the available data and our algorithm identifies it successfully.

\subsection{Learning from private or decentralized data}
\label{sec:discussion_privacy}
The described algorithm is straightforward to implement on top of any standard learning procedure, when the data from all $N$ sources is directly available to the learner. We now discuss how we can learn robustly in cases where the sources cannot fully reveal their data. There are many applications where such a situation can arise. For example, this can be due to privacy reasons in the case of medical and biological data or to communication costs and storage limitations in the case of distributed learning \cite{mcmahan2017communication}.

Here we focus on ways to compute the discrepancies under such constraints. Once this is done, the vector $\alpha$ can be computed easily and then any standard distributed training procedure, \eg \cite{dean2012large, mcmahan2017communication}, can be used to obtain the $\alpha$-weighted empirical loss minimizer. Standard approaches in distributed learning only require the exchange of gradients of minibatches with respect to the current state of the model between the data sources and the central server, so in particular the actual local datasets are never observed by the learner. In cases when the gradients may reveal sensitive information about the data, secure aggregation \cite{bonawitz2017practical} or other privacy-preserving distributed learning methods \cite{shokri2015privacy} can be used on top to ensure privacy. 

We distinguish two cases, depending on whether the reference dataset can be shared with the sources.

\textbf{Case 1: the reference dataset is available to all nodes.} If the reference dataset can be shared with the sources without privacy and communication complications, the discrepancies can be estimated \textit{locally on every source, in parallel}. If necessary, the computational protocol can be executed via a trusted computation method \cite{trustedcomp}, for example by using Software Guard Extensions (SGX extensions) \cite{mckeen2016intel}, to ensure the correctness of the procedure. The discrepancies alone can then be sent to the learner and the algorithm proceeds as described above. This approach ensures the privacy of the local datasets and allows for all discrepancies to be computed in parallel.

\textbf{Case 2: the reference dataset can not be shared.} In this case the learner can still compute the empirical discrepancies without observing the data from the sources directly, by using a gradient-based optimization procedure. This is because the function inside the supremum in (\ref{eqn:defn_of_emp_disrepancy}) decomposes into a term depending only on the reference dataset and a term depending only on the data of the source. Therefore, each discrepancy can be estimated by using a sequence of queries to the source about the gradient of a minibatch from its data with respect to a current candidate for the predictor achieving the supremum.

\section{Experiments}
\label{sec:experiments}
\subsection{Method and baselines}
We perform two large sets of experiments, following the setup considered in our paper. We train our algorithm on the data from all sources, including the reference dataset. The hyperparameter $\lambda$ is selected by 5-fold cross-validation \textit{on the trusted data}. The prediction tasks we consider here are binary classification problems with the $0/1$-loss, so we compute the empirical discrepancies by approximately solving the optimization problem (\ref{eqn:defn_of_emp_disrepancy}) as follows. Given the two datasets $S_i$ and $S_T$, the binary labels of one of them are flipped. The optimization can then be reduced to an empirical risk minimization problem that we solve using a standard convex relaxation approach. We refer to the supplementary material for a more formal description.

We compare the performance of our algorithm to the two naive approaches: training on the reference dataset only (corresponding to $\lambda = 0$ in our algorithm; denoted as \textit{"Reference only"} in the plots and tables) and merging the sources and training on all the data (corresponding to $\lambda \rightarrow \infty$; referred to as \textit{"All data"} in the plots and tables). All three methods use linear predictors and are trained by regularized logistic regression. The regularization parameter is always selected by 5-fold cross-validation on the reference data. The learned models are then evaluated on held out test data. 

Our aim is to test whether the proposed algorithm successfully leverages information from the sources, while being robust to various perturbations in the distributions of the local datasets, and whether exploiting the multi-source structure of the data gives any improvement over the two standard learning procedures. We also compare the performance of our algorithm to the following robust learning baselines. 

\textbf{Robust aggregation of local models.} We consider two recently proposed approaches for robust distributed learning. Following \cite{feng2014distributed}, one baseline learns a separate linear model based on each of the source datasets. The final linear predictor is then constructed as the geometric median of these locally learned weight vectors. Another baseline, inspired by \cite{pmlr-v80-yin18a}, takes the component-wise median instead. Thirdly, based on the locally learned models all $N$ estimates for the probability that a test point belongs to a certain class are computed and the final prediction for the label of that point is obtained by taking the median of these probabilities and thresholding it (referred to as "Median of probs" in the plots and tables). All these baselines aim at learning a robust ensemble of local models.

\textbf{Robust logistic regression.} We use the method of \cite{pregibon1982resistant}, based on the minimization of a Huber-type modification of the logistic loss. Specifically, the method minimizes the following robust loss function, instead of the classic logistic loss:
\begin{equation*}
    L\left(\boldsymbol{w}, \boldsymbol{x}, y\right) =
    \begin{cases}
      \log(1+ e^{-y\boldsymbol{w}^{\text{T}}\boldsymbol{x}}), \text{ if }\ \log(1+ e^{-y\boldsymbol{w}^{\text{T}}\boldsymbol{x}}) \leq c \\
      2\sqrt{c\log(1+ e^{-y\boldsymbol{w}^{\text{T}}\boldsymbol{x}})} - c, \text{ otherwise}
    \end{cases}
\end{equation*}
In our experiments, we use the recommended threshold value of $c = 1.345^2$, under which the estimate of the linear predictor has been shown to achieve a $95\% $ asymptotic relative efficiency \cite{pregibon1982resistant}. We also include a regularization term here and learn the regularization parameter by 5-fold cross-validation on the reference data. This baseline is an example of learning robustly on the whole dataset.

\textbf{Batch normalization.} Inspired by the success of \textit{batch normalization} in deep learning \cite{ioffe2015batch}, we compute the mean and standard deviation of the data at each source separately. We then subtract from each data point the mean and divide by the standard deviation of its corresponding dataset. We do the same for the reference data. We then merge all data together and train a logistic regression model with a regularization term. Finally, at test time every input is preprocessed by subtracting the mean and dividing by the standard deviation of the reference dataset, before applying the classifier. This approach aims at increasing robustness to source-specific biases.

\subsection{Amazon Products data}
\label{sec:experiments_products}

Our first set of experiments is on the "Multitask dataset of product reviews"\footnote{\href{http://cvml.ist.ac.at/productreviews/}{http://cvml.ist.ac.at/productreviews/}} \cite{PenLam17}, containing customer reviews for 957 Amazon products from the "Amazon product data" \cite{mcauley2015inferring, mcauley2015image}, together with a binary label indicating whether each review is positive or negative. All reviews in the data set are represented via 25-dimensional feature vectors, obtained by computing a GloVe word embedding \cite{pennington2014glove} and applying the sentence embedding procedure of \cite{arora2016simple}. We treat the classification of a review as positive or negative as a separate prediction task for each of the products, resulting in a total of 957 input-output distributions.

As a first, illustrative, experiment, we chose 20 books and 20 other, purposely different, products (\eg USB drives, mobile apps, meal replacement products). For simplicity, we refer to these additional products as "non-books". Intuitively, when learning to classify book reviews and given access to reviews from both some books and some non-books, a good learning algorithm will be able to leverage all this information, while being robust to the potentially misleading data coming from the less relevant products.

We randomly sample one of the books and 300 positive and 300 negative reviews for it. Out of those, 100 randomly selected reviews are made available to the learner as a reference dataset. The 500 remaining reviews from the product are used for testing. For a given value of $n \in \{0, 1,\ldots, 10\}$ the learner also has access to 100 labeled reviews from each of $10-n$ other randomly selected books and from each of $n$ randomly selected non-books. Our algorithm, as well as all baselines, are trained on this available data and the learned predictors are evaluated on the test set for the target product. For each $n$, we repeat this experiment 1000 times.

The results are plotted in Figure \ref{fig:results_products}. The $x$-axis corresponds to the number $n$ of non-books and the $y$-axis gives the average classification error. The error bars correspond to the standard errors of the mean estimates. We see that our method (green) performs uniformly better than the naive approaches of training on the reference dataset from the target product only (red) and training by merging all data together (blue). When reviews from many books are available, our algorithm is able to use this additional information even better than the model learned on all data. As the proportion of non-books increases, the performance of the second approach degrades, confirming the intuition that the reviews for the non-books provide less useful information for the target task. On the other hand, our algorithm successfully incorporates the information from the useful sources only, converging to the performance of the model learned on the reference data as all additional sources become non-books.

Our algorithm also outperforms all baselines. The batch normalization approach appears to reduce the effect of irrelevant sources, but its performance degrades as $n \rightarrow 10$. The median-based approaches perform reasonably when at most half of the sources are non-books, but eventually become worse than the other methods. The component-wise median and the robust loss baselines were excluded from the plot for clarity, as they performed uniformly worse than the other baselines, ranging in average classification error from 0.338 to 0.375 and from 0.348 to 0.372 respectively. Note that the robust loss function of \cite{pregibon1982resistant} is non-convex, so the poor performance of this baseline is presumably due to failure of the gradient descent optimization procedure to converge to a good local minimum.

\begin{figure}
\centering
  \includegraphics[width=.95\linewidth]{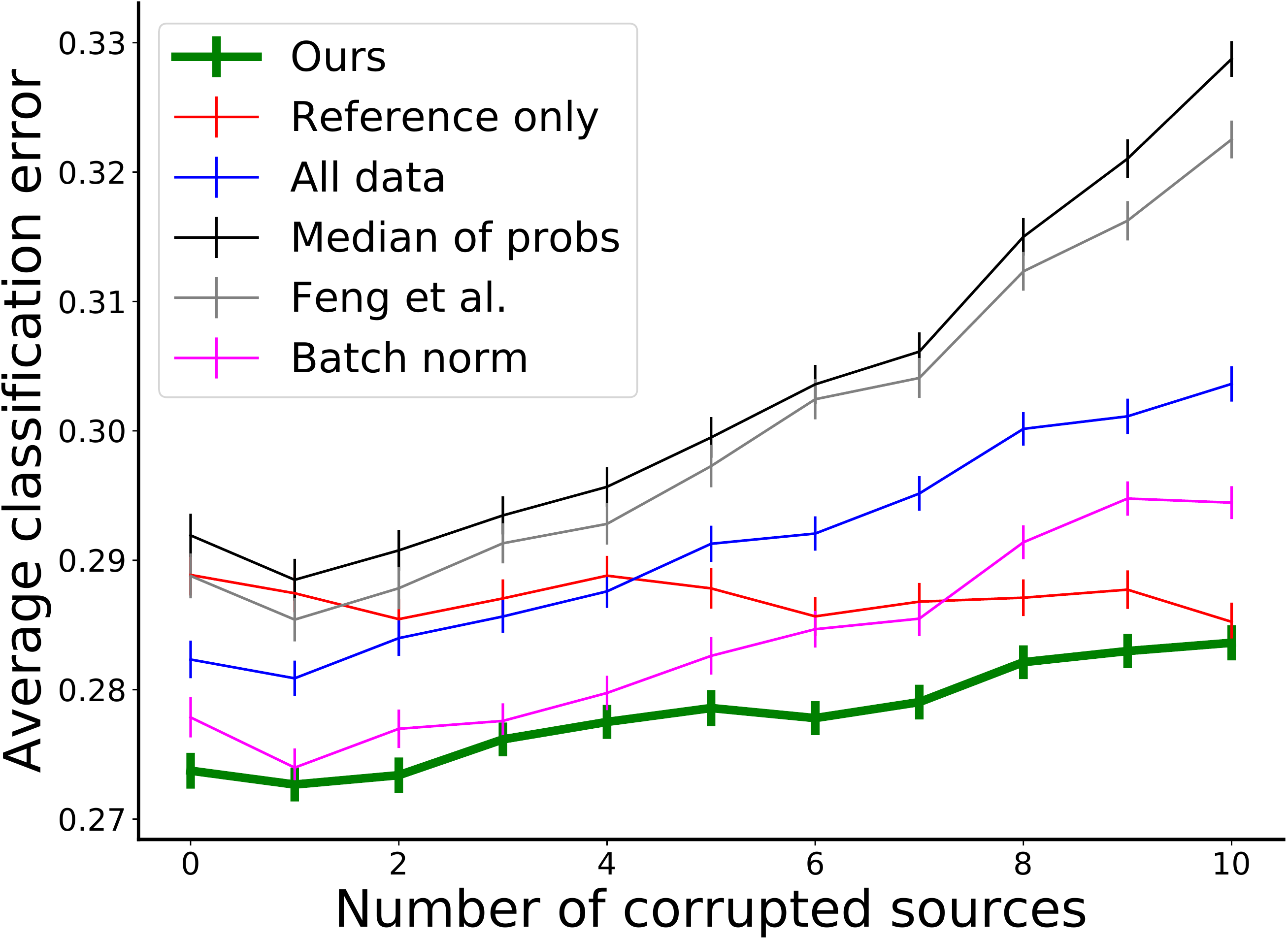}
  \caption{Results from the experiments on 20 books and 20 other products from the "Multitask dataset of product reviews". The $x$-axis gives the number $n$ of non-books in an experiment and the $y$-axis - the mean classification error. Error bars give the standard error of the estimates.}
  \label{fig:results_products}
\end{figure}

\begin{table}[t]
\caption{Results from the experiment on all 957 products.}
\label{table:products_table}
\centering\setlength{\tabcolsep}{1.5pt}
 \begin{tabular}{| c | c |} 
 \hline
 Algorithm & Mean classification error \\ 
 \hline\hline 
\textbf{Ours} & $\mathbf{0.289 \pm 0.0016 } $\\
Reference only &  $0.301 \pm 0.0019$\\
All data & $0.312 \pm 0.0017$\\
Median of probs. & $0.325 \pm 0.0021$  \\
Geom.median~\cite{feng2014distributed} & $0.329 \pm 0.0021$ \\
Comp.median~\cite{pmlr-v80-yin18a} & $0.329 \pm 0.0021$\\
Robust loss~\cite{pregibon1982resistant} & $0.353 \pm 0.0021$\\
Batch norm & $0.298 \pm 0.0016$ \\
 \hline
 \end{tabular}
\end{table}

Additionally, we performed an experiment on the set of all 957 products. With every product as a prediction task, we randomly selected 100 reviews from it as a reference dataset, leaving 500 for testing. An additional set of 100 labeled reviews were available from every other product. The algorithms were trained on all available data and evaluated on the test set. The average classification errors achieved by the algorithms are presented in Table \ref{table:products_table}, together with the standard errors of those estimates. We see in particular that our algorithm successfully uses the information from multiple sources to achieve the best overall performance.

\begin{figure*}[h]
\begin{subfigure}{.33\textwidth}
  \centering
  \includegraphics[width=.95\linewidth]{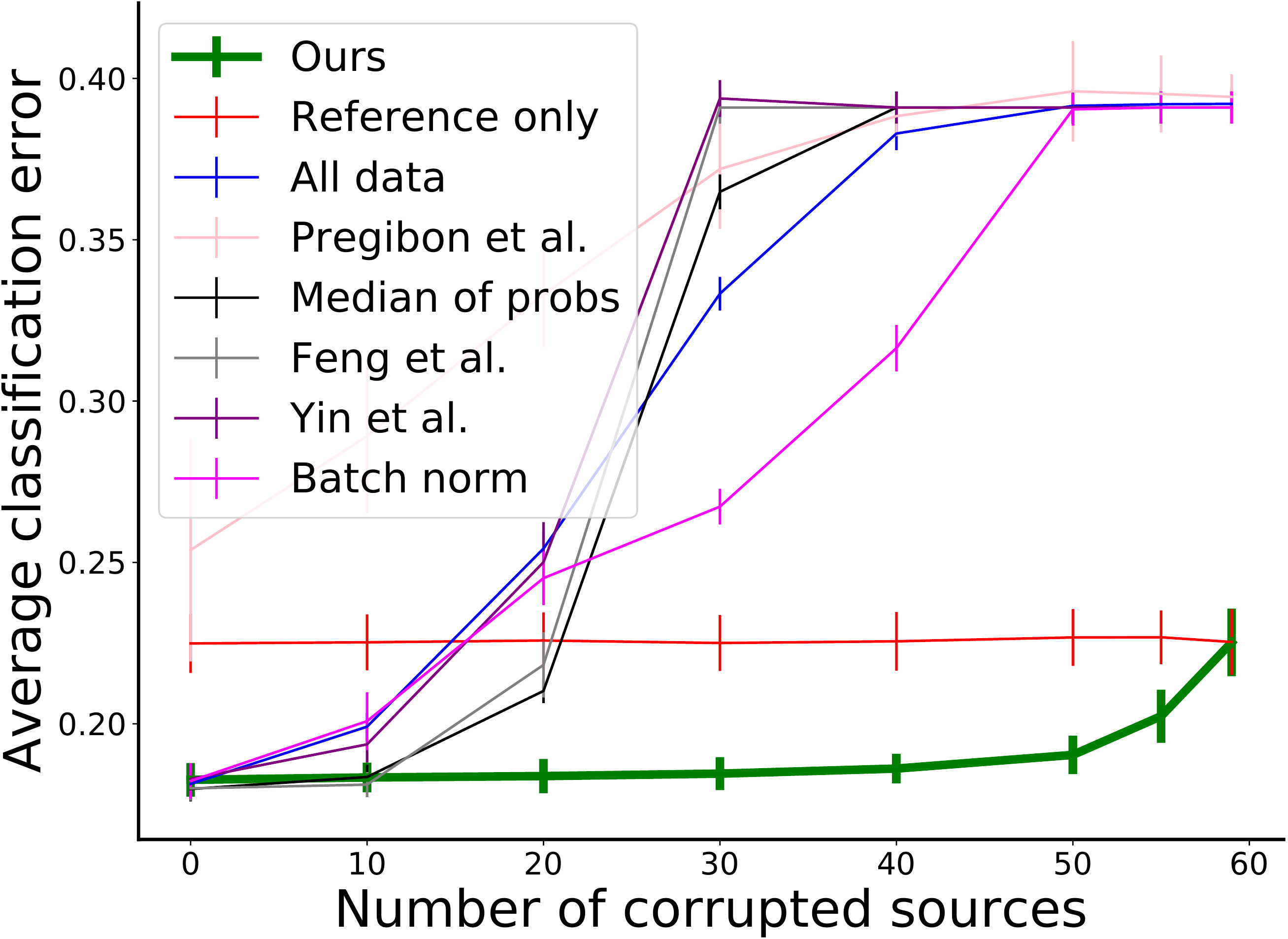}
  \caption{Label bias}
  \label{fig:animals_enforce_label}
\end{subfigure}%
\hfill
\begin{subfigure}{.33\textwidth}
  \centering
  \includegraphics[width=.95\linewidth]{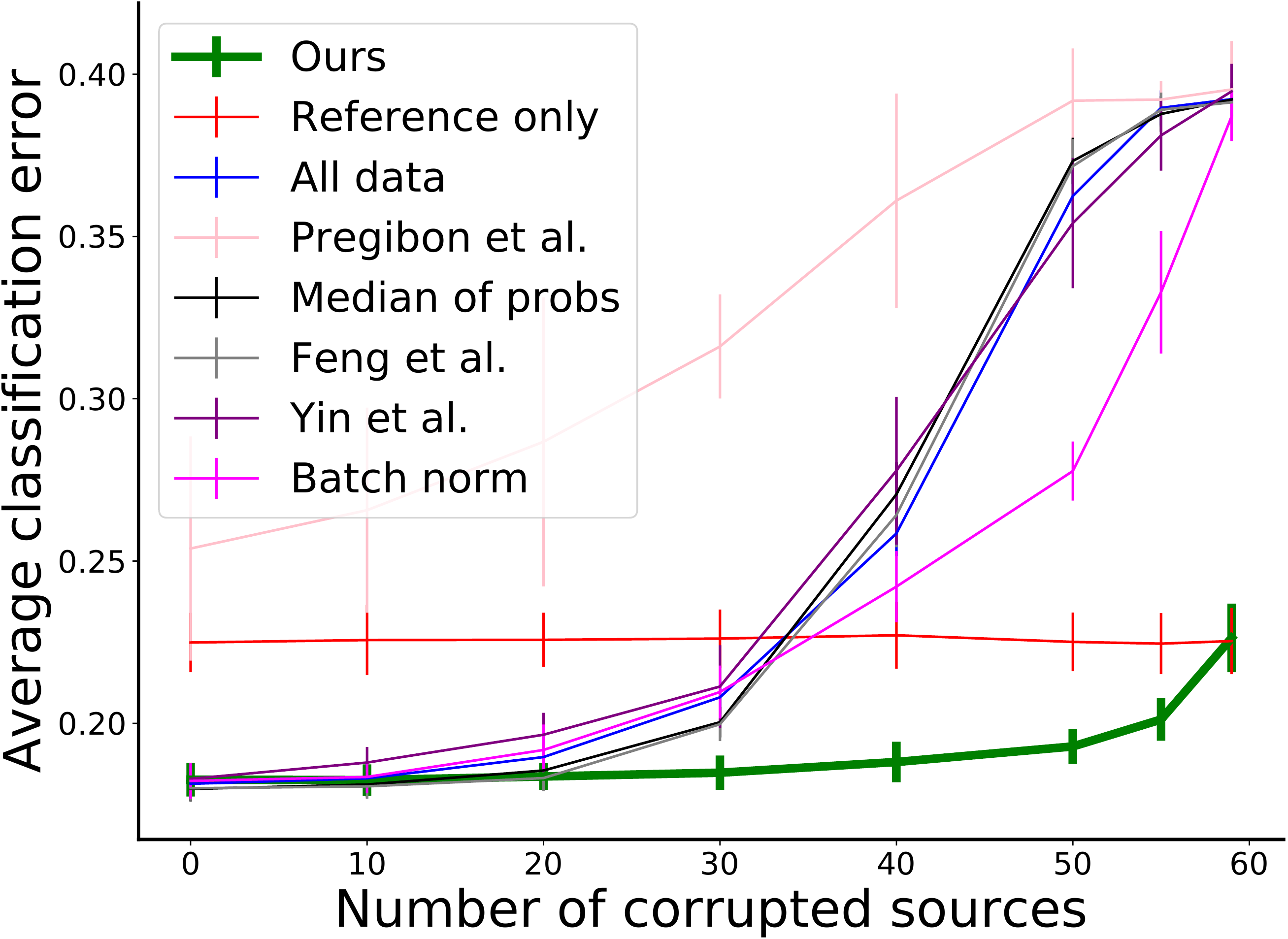}
  \caption{Shuffled labels}
  \label{fig:animals_shuffle_labels}
\end{subfigure}%
\hfill
\begin{subfigure}{.33\textwidth}
  \centering
  \includegraphics[width=.95\linewidth]{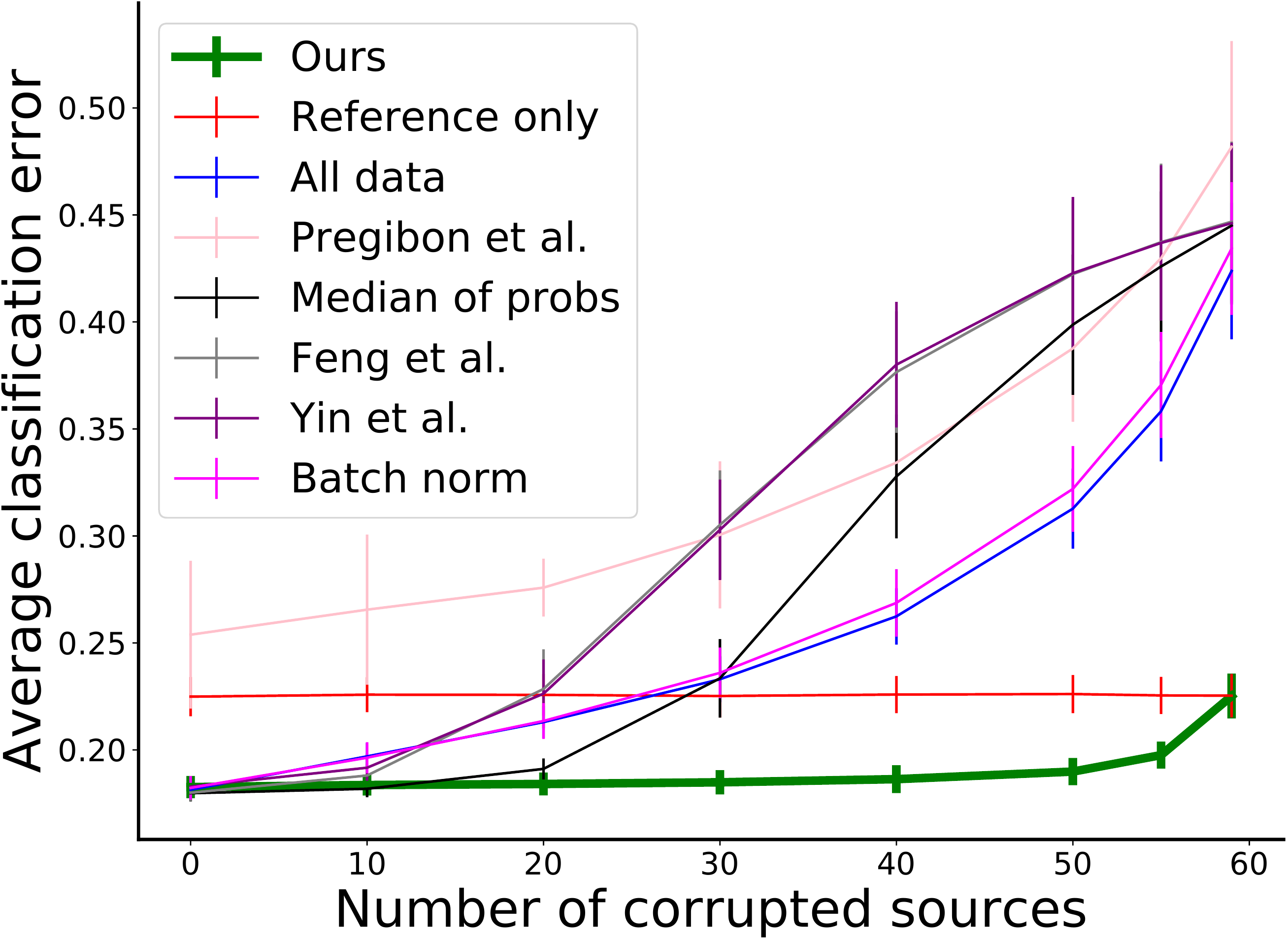}
  \caption{Shuffled features}
  \label{fig:animals_shuffle_inputs}
\end{subfigure}
\vskip\baselineskip
\begin{subfigure}{.33\textwidth}
  \centering
  \includegraphics[width=.95\linewidth]{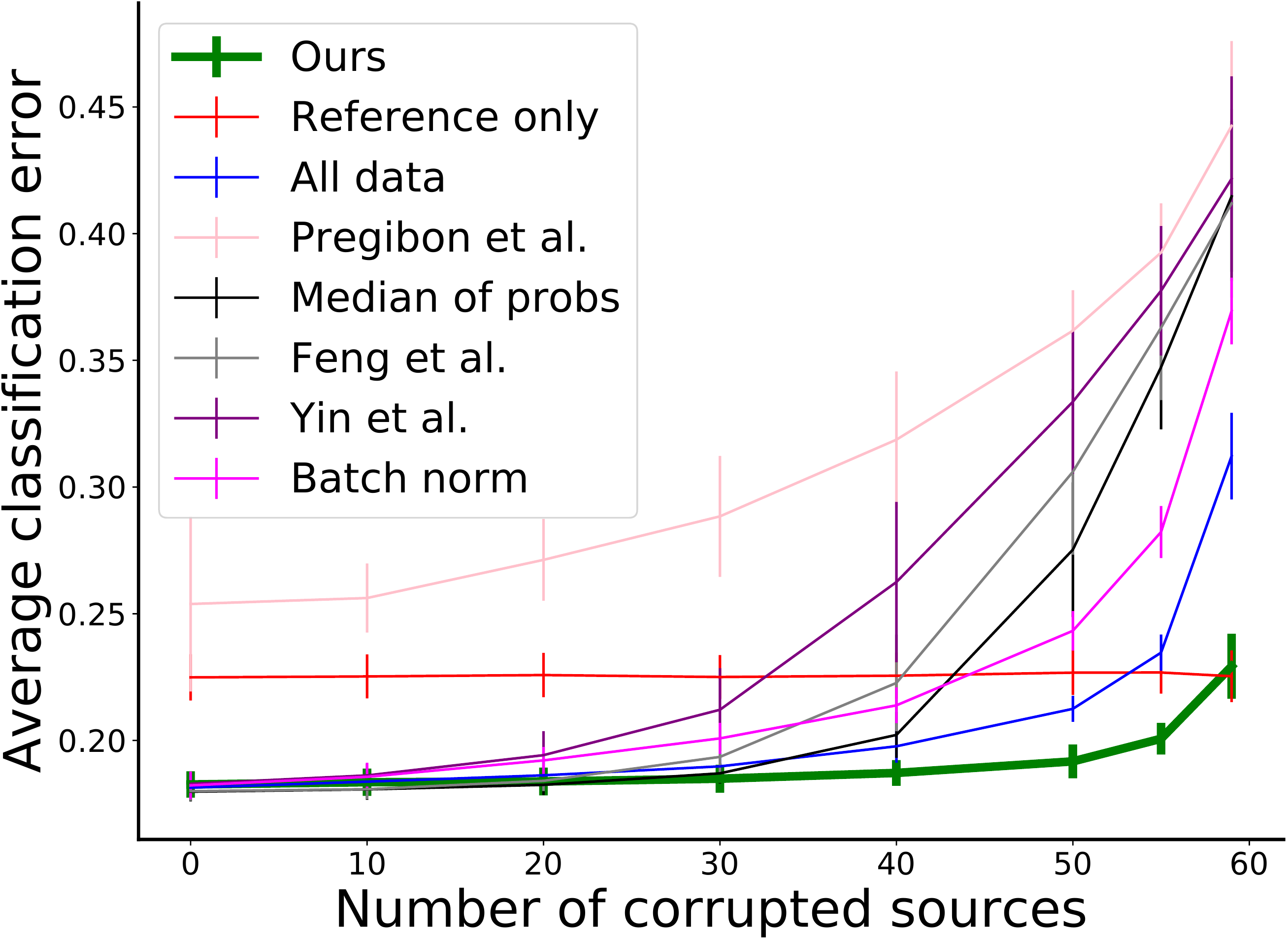}
  \caption{Blurred images}
  \label{fig:animals_blured}
\end{subfigure}%
\hfill
\begin{subfigure}{.33\textwidth}
  \centering
  \includegraphics[width=.95\linewidth]{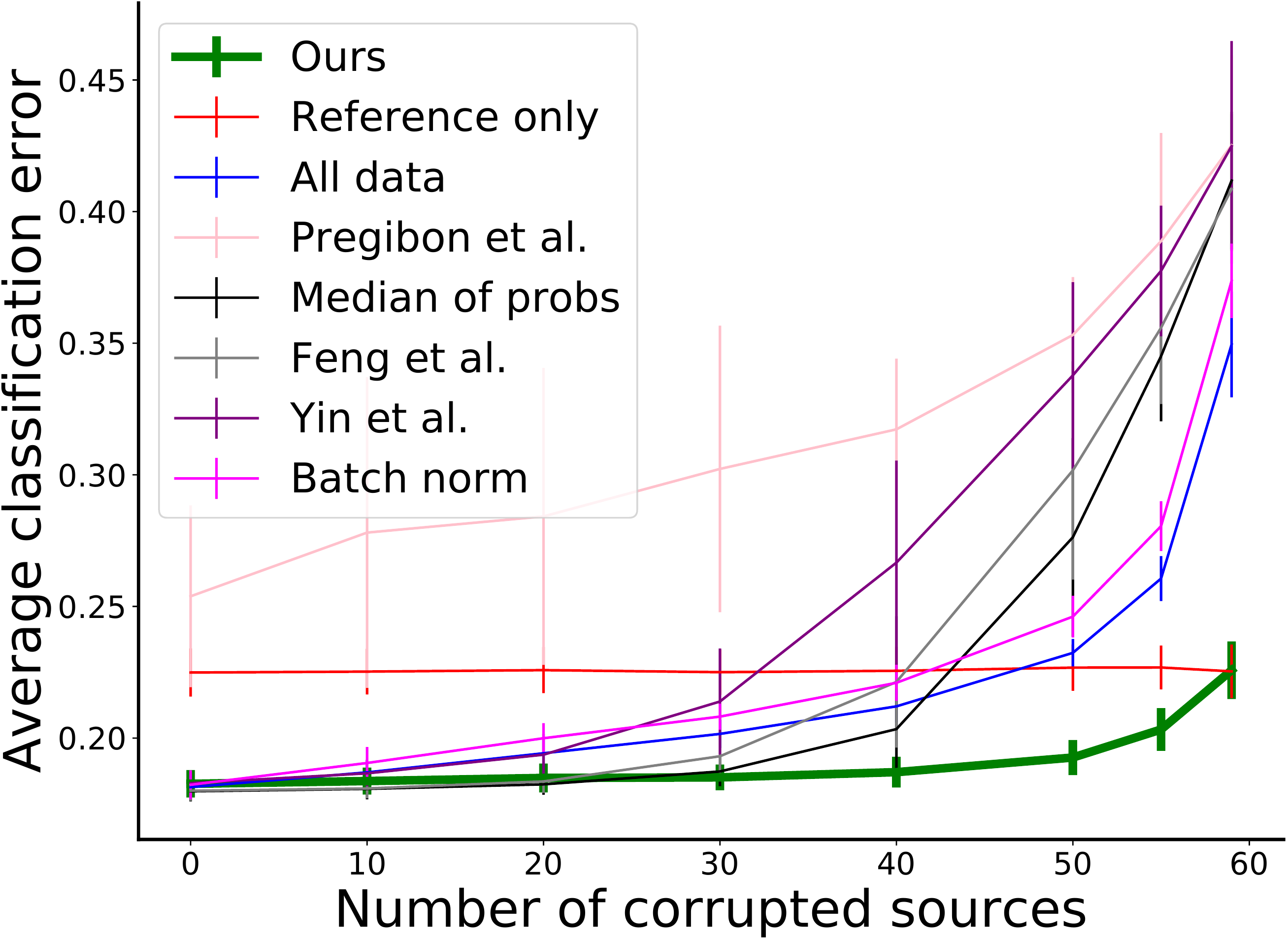}
  \caption{Dead pixels}
  \label{fig:animals_pixels}
\end{subfigure}%
\hfill
\begin{subfigure}{.33\textwidth}
  \centering
  \includegraphics[width=.95\linewidth]{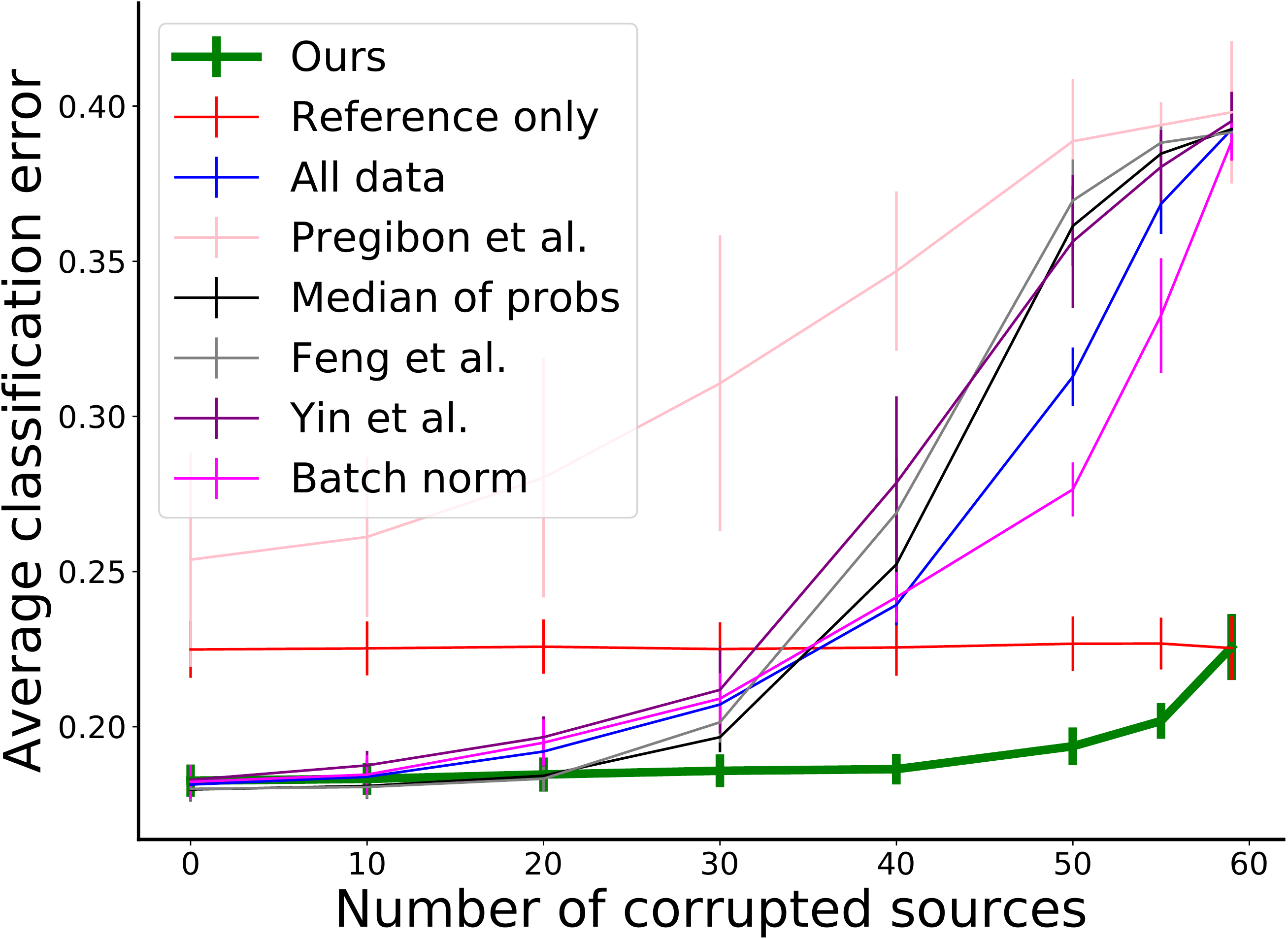}
  \caption{RGB channels swapped}
  \label{fig:animals_RGB}
\end{subfigure}
\caption{Results for the attribute "black" from the Animals with Attributes 2 dataset. Each plot corresponds to a different contamination type. The $x$-axis gives the number $n$ of corrupted sources and the $y$-axis gives the average classification error of the algorithms, achieved over 100 different runs. Error bars correspond to the standard deviation around those means.}
\label{fig:animals_plots}
\end{figure*}

\subsection{Animals with Attributes 2}
\label{sec:experiments_animals} 
The Animals with Attributes 2 dataset \cite{xian2018zero} contains 37322 images of 50 animal classes. The classes are aligned to $85$ binary attributes, \eg color, habitat and others, via a class-attribute binary matrix, indicating whether an animal possesses each feature. This results in a total of $85$ different binary prediction tasks of identifying whether an animal on a given image possesses a certain attribute or not.

Feature representations of the images are obtained via the following procedure. We use a ResNet50 network \cite{he2016deep}, pretrained \footnote{We use a pretrained model from the TensorNets package, \href{https://github.com/taehoonlee/tensornets}{https://github.com/taehoonlee/tensornets}.} on ImageNet \cite{russakovsky2015imagenet}, to obtain feature representations of the ImageNet data and reduce their dimension to 100 by PCA. Finally, for each image in the Animals with Attributes 2 dataset, we compute the ResNet50 feature representation and apply the PCA projection pre-learned on ImageNet.

We perform an independent set of experiments for each attribute and for various types and levels of corruption of the data sources. In each run, we randomly split the data into 60 groups of 500 images, with the remaining 7322 images left out for testing. One of the groups is selected at random as the clean reference dataset available to the learner. The remaining 59 groups correspond to the data sources, some of which provide low-quality or corrupted data. We consider six different types of corruptions. Three act on the labels or the feature representations directly and the next three are synthetic modifications of the images themselves. In the second case, the corresponding images are manipulated before the feature representations are extracted.

\begin{table*}[h!]
\centering\setlength{\tabcolsep}{3.3pt}
 \begin{tabular}{|c || c | c | c | c | c | c | c | c |} 
 \hline
 \backslashbox{Baseline}{$n$} & $n = 0$ & $n = 10$ & $n = 20$ & $n = 30$ & $n = 40$ & $n = 50$ & $n = 55$ & $n = 59$ \\ 
 \hline\hline 
Reference only & 84/1/0 & 505/5/0 & 497/13/0 & 487/23/0 & 475/35/0 & 442/68/0 & 325/185/0 & 0/510/0 \\
All data & 0/85/0 & 115/395/0 & 267/243/0 & 370/140/0 & 438/72/0 & 468/42/0 & 479/31/0 & 484/26/0 \\
Med of probs. & 9/76/0 & 47/463/0 & 172/338/0 & 336/174/0 & 469/41/0 & 504/6/0 & 502/8/0 & 499/11/0 \\
Geom.med~\cite{feng2014distributed} & 8/77/0 & 32/478/0 & 110/400/0 & 338/172/0 & 457/53/0 & 504/6/0 & 502/8/0 & 497/13/0 \\
Comp.med~\cite{pmlr-v80-yin18a} & 14/71/0 & 179/331/0 & 390/120/0 & 432/78/0 & 472/38/0 & 502/8/0 & 503/7/0 & 497/13/0 \\
Robust loss~\cite{pregibon1982resistant} & 55/30/0 & 308/202/0 & 361/149/0 & 416/94/0 & 437/73/0 & 455/55/0 & 470/40/0 & 485/25/0 \\
Batch norm & 0/85/0 & 107/403/0 & 317/193/0 & 416/94/0 & 446/63/1 & 478/32/0 & 487/23/0 & 482/28/0 \\
 \hline
 \end{tabular}
 \caption{Summary of the results from the Animals with Attributes 2 experiments, over all 85 prediction tasks and all 6 types of corruption. Given a number of corrupted sources $n$ (columns) and a baseline (rows), we report values in the form A/B/C, where A is the number of times that our method performed significantly better than the corresponding baseline, B is the number of times it performed equally well and C is the number of times it performed significantly worse, summed over the various types of corruptions and all attributes. More details are provided in the main body of the text.}
 \label{table:animals_tables}
\end{table*}

\begin{itemize}
\item Label bias: The labels of all (corrupted) samples are switched to class $1$.
\item Shuffled labels: The labels of all samples are shuffled randomly, separately in each corrupted source.
\item Shuffled features: Given a permutation of the indexes between 1 and 100, the features of all samples are shuffled according to it.
\item Blurred images: Each image is blurred by filtering with a Gaussian kernel with standard deviation $\sigma = 6$.
\item Dead pixels: In each image a random $30\%$ of the pixels are set to pure black or white.
\item RGB channels swapped: The values in the red and the blue color channels of each image are swapped.
\end{itemize}

Given an attribute, a type of corruption and a value of $n\in \{0, 10, 20, 30, 40, 50, 55, 59\}$, the data is split randomly, as described above, and the samples of $n$ randomly chosen sources are corrupted. Our algorithm, as well as all baselines, then learn a model based on the resulting data and the performance of the obtained predictors is evaluated on the test data. For any combination of target attribute, corruption strategy and value of $n$, the experiment is repeated 100 times with a different random seed to obtain error estimates.

The results for the first attribute from the Animals with Attributes 2 data ("black") are given in Figure \ref{fig:animals_plots}. Each plot corresponds to a different type of contamination. The $x$-axis gives the number of sources providing corrupted data and the $y$-axis corresponds to the average error that an algorithm achieved on the test set, over the 100 runs for each experimental setup. The error bars give the standard deviation around this average.

Our algorithm (green) performs at least as well as or strictly better than all baselines, for \textit{any} type of corruption and \textit{any} proportion of corrupted sources. When all sources provide clean data, the performance of our method matches the one of the classic regularized logistic regression approach on \iid data (blue). As the number of corrupted sources increases, the performance of all baselines gradually degrades, while our algorithm is able to leverage the remaining clean data and suppress the effect of the corruptions. The median-based baselines perform reasonably when less than half of the sources are corrupted, but fail for larger proportions. The robust logistic regression baseline performs poorly, again likely due to the non-convexity of the loss function. As all sources become unreliable, our method performs as well as the approach of learning from the reference dataset only, which is indeed optimal since all other data is corrupted.

We summarize the results from all attributes in Table \ref{table:animals_tables}. For any number of corrupted sources $n$ (columns), we compare our method to the performance of each baseline (rows). We report values in the form A/B/C, where A is the number of times that our method performed significantly better than the corresponding baseline, B is the number of times it performed equally well and C is the number of times it performed significantly worse, summed over the various types of corruptions and all attributes. For a fixed type of corruption and attribute, we say that one method performs significantly better than another over the set of 100 runs with this setup, if the difference in the average performance of the two models is larger than the sum of the standard deviations around those means (that is, if the error bars, as in Figure \ref{fig:animals_plots}, do not intersect). 

The results in Table \ref{table:animals_tables} show that our method \textit{performs significantly better than all baselines for many types of corruption and many values of $n$}, especially for high levels of contamination, while \textit{essentially never performing significantly worse than any baseline}. Tables with a more detailed breakdown, depending on the type of corruption, as well as results for lower levels of contamination per source, are deferred to the supplementary material.

\section{Conclusion}
\label{sec:conclusion}
We introduce an algorithm for learning from data provided by multiple untrusted sources. It incorporates information from all of them, while being robust to arbitrary corruptions and manipulations of the data. 
By making use of the grouped structure of the task and a reference dataset, the method is able 
to successfully learn even if more than half of the available data is 
corrupted or uninformative.
Our method is theoretically justified and easy to implement, even in cases when the data is decentralized and/or private. We demonstrated its effectiveness through two sets of extensive experiments, showing its superior performance to all baselines, for various levels and types of corruption.

In our experiments we observed that a relatively small clean dataset was enough to protect the learning from the effects of corrupted data. Quantifying the trade-off between the size of the reference dataset and the gains of our algorithm in terms of achieved test-time performance is thus an interesting and promising direction for future work.

\section*{Acknowledgements}
We thank Dan Alistarh for helpful suggestions regarding the discussion in Section \ref{sec:discussion_privacy}. This work was in parts funded by the European Research Council under the European Union’s Seventh Framework Programme (FP7/2007-2013)/ERC grant agreement no 308036. This project has received funding from the European Union’s Horizon 2020 research and innovation programme under the Marie Skłodowska-Curie Grant Agreement No. 665385.

\bibliography{ms}

\begin{thebibliography}{55}
\providecommand{\natexlab}[1]{#1}
\providecommand{\url}[1]{\texttt{#1}}
\expandafter\ifx\csname urlstyle\endcsname\relax
  \providecommand{\doi}[1]{doi: #1}\else
  \providecommand{\doi}{doi: \begingroup \urlstyle{rm}\Url}\fi

\bibitem[Alistarh et~al.(2018{\natexlab{a}})Alistarh, Allen-Zhu, and
  Li]{NIPS2018_7712}
Alistarh, D., Allen-Zhu, Z., and Li, J.
\newblock Byzantine stochastic gradient descent.
\newblock In \emph{Conference on Neural Information Processing Systems
  (NeurIPS)}, 2018{\natexlab{a}}.

\bibitem[Alistarh et~al.(2018{\natexlab{b}})Alistarh, De~Sa, and
  Konstantinov]{Alistarh:2018:CSG:3212734.3212763}
Alistarh, D., De~Sa, C., and Konstantinov, N.
\newblock The convergence of stochastic gradient descent in asynchronous shared
  memory.
\newblock In \emph{ACM Symposium on Principles of Distributed Computing}, PODC,
  2018{\natexlab{b}}.

\bibitem[Arora et~al.(2017)Arora, Liang, and Ma]{arora2016simple}
Arora, S., Liang, Y., and Ma, T.
\newblock A simple but tough-to-beat baseline for sentence embeddings.
\newblock In \emph{International Conference on Learning Representations
  (ICLR)}, 2017.

\bibitem[Awasthi et~al.(2017)Awasthi, Blum, Haghtalab, and
  Mansour]{awasthi2017efficient}
Awasthi, P., Blum, A., Haghtalab, N., and Mansour, Y.
\newblock Efficient {PAC} learning from the crowd.
\newblock \emph{Conference on Computational Learning Theory (COLT)}, 2017.

\bibitem[Ben-David et~al.(2010)Ben-David, Blitzer, Crammer, Kulesza, Pereira,
  and Vaughan]{ben2010theory}
Ben-David, S., Blitzer, J., Crammer, K., Kulesza, A., Pereira, F., and Vaughan,
  J.~W.
\newblock A theory of learning from different domains.
\newblock \emph{Machine Learning}, 79\penalty0 (1-2):\penalty0 151--175, 2010.

\bibitem[Bi et~al.(2014)Bi, Wang, Kwok, and Tu]{bi2014learning}
Bi, W., Wang, L., Kwok, J.~T., and Tu, Z.
\newblock Learning to predict from crowdsourced data.
\newblock In \emph{UAI}, pp.\  82--91, 2014.

\bibitem[Biggio et~al.(2012)Biggio, Nelson, and Laskov]{Biggio}
Biggio, B., Nelson, B., and Laskov, P.
\newblock Poisoning attacks against support vector machines.
\newblock In \emph{International Conference on Machine Learing (ICML)}, 2012.

\bibitem[Blanchard et~al.(2017)Blanchard, Guerraoui, Stainer,
  et~al.]{blanchard2017machine}
Blanchard, P., Guerraoui, R., Stainer, J., et~al.
\newblock Machine learning with adversaries: Byzantine tolerant gradient
  descent.
\newblock In \emph{Conference on Neural Information Processing Systems (NIPS)},
  2017.

\bibitem[Bonawitz et~al.(2017)Bonawitz, Ivanov, Kreuter, Marcedone, McMahan,
  Patel, Ramage, Segal, and Seth]{bonawitz2017practical}
Bonawitz, K., Ivanov, V., Kreuter, B., Marcedone, A., McMahan, H.~B., Patel,
  S., Ramage, D., Segal, A., and Seth, K.
\newblock Practical secure aggregation for privacy-preserving machine learning.
\newblock In \emph{Proceedings of the 2017 ACM SIGSAC Conference on Computer
  and Communications Security}, 2017.

\bibitem[Bousquet et~al.(2004)Bousquet, Boucheron, and
  Lugosi]{bousquet2004introduction}
Bousquet, O., Boucheron, S., and Lugosi, G.
\newblock Introduction to statistical learning theory.
\newblock In \emph{Advanced lectures on machine learning}, pp.\  169--207.
  Springer, 2004.

\bibitem[Charikar et~al.(2017)Charikar, Steinhardt, and
  Valiant]{charikar2017learning}
Charikar, M., Steinhardt, J., and Valiant, G.
\newblock Learning from untrusted data.
\newblock In \emph{ACM SIGACT Symposium on Theory of Computing}, 2017.

\bibitem[Chen et~al.(2009)Chen, Mitchell, and Martin]{trustedcomp}
Chen, L., Mitchell, C.~J., and Martin, A.~P. (eds.).
\newblock \emph{Trusted Computing, Second International Conference, Trust 2009,
  Oxford, UK, April 6-8, 2009, Proceedings}, volume 5471 of \emph{Lecture Notes
  in Computer Science}, 2009.

\bibitem[Crammer et~al.(2008)Crammer, Kearns, and Wortman]{crammer2008learning}
Crammer, K., Kearns, M., and Wortman, J.
\newblock Learning from multiple sources.
\newblock \emph{Journal of Machine Learning Research (JMLR)}, 9\penalty0
  (Aug):\penalty0 1757--1774, 2008.

\bibitem[De~Sa et~al.(2015)De~Sa, Zhang, Olukotun, R\'{e}, and
  R\'{e}]{NIPS2015_5717}
De~Sa, C.~M., Zhang, C., Olukotun, K., R\'{e}, C., and R\'{e}, C.
\newblock Taming the wild: A unified analysis of hogwild-style algorithms.
\newblock In \emph{Conference on Neural Information Processing Systems (NIPS)}.
  2015.

\bibitem[Dean et~al.(2012)Dean, Corrado, Monga, Chen, Devin, Mao, Senior,
  Tucker, Yang, Le, et~al.]{dean2012large}
Dean, J., Corrado, G., Monga, R., Chen, K., Devin, M., Mao, M., Senior, A.,
  Tucker, P., Yang, K., Le, Q.~V., et~al.
\newblock Large scale distributed deep networks.
\newblock In \emph{Conference on Neural Information Processing Systems (NIPS)},
  2012.

\bibitem[Diakonikolas et~al.(2016)Diakonikolas, Kamath, Kane, Li, Moitra, and
  Stewart]{diakonikolas2016robust}
Diakonikolas, I., Kamath, G., Kane, D.~M., Li, J., Moitra, A., and Stewart, A.
\newblock Robust estimators in high dimensions without the computational
  intractability.
\newblock In \emph{Foundations of Computer Science (FOCS)}, pp.\  655--664.
  IEEE, 2016.

\bibitem[Feng(2017)]{feng2017fundamental}
Feng, J.
\newblock On fundamental limits of robust learning.
\newblock \emph{arXiv preprint arXiv:1703.10444}, 2017.

\bibitem[Feng et~al.(2014)Feng, Xu, and Mannor]{feng2014distributed}
Feng, J., Xu, H., and Mannor, S.
\newblock Distributed robust learning.
\newblock \emph{arXiv preprint arXiv:1409.5937}, 2014.

\bibitem[Fung et~al.(2018)Fung, Yoon, and Beschastnikh]{fung2018mitigating}
Fung, C., Yoon, C.~J., and Beschastnikh, I.
\newblock Mitigating sybils in federated learning poisoning.
\newblock \emph{arXiv preprint arXiv:1808.04866}, 2018.

\bibitem[He et~al.(2016)He, Zhang, Ren, and Sun]{he2016deep}
He, K., Zhang, X., Ren, S., and Sun, J.
\newblock Deep residual learning for image recognition.
\newblock In \emph{Conference on Computer Vision and Pattern Recognition
  (CVPR)}, 2016.

\bibitem[Hendrycks \& Gimpel(2017)Hendrycks and Gimpel]{hendrycks2017baseline}
Hendrycks, D. and Gimpel, K.
\newblock Improving the generalization of adversarial training with domain
  adaptation.
\newblock In \emph{International Conference on Learning Representations
  (ICLR)}, 2017.

\bibitem[Hendrycks et~al.(2018)Hendrycks, Mazeika, Wilson, and
  Gimpel]{NIPS2018_8246}
Hendrycks, D., Mazeika, M., Wilson, D., and Gimpel, K.
\newblock Using trusted data to train deep networks on labels corrupted by
  severe noise.
\newblock In \emph{Conference on Neural Information Processing Systems
  (NeurIPS)}, 2018.

\bibitem[Hoffman et~al.(2018)Hoffman, Mohri, and Zhang]{NIPS2018_8046}
Hoffman, J., Mohri, M., and Zhang, N.
\newblock Algorithms and theory for multiple-source adaptation.
\newblock In \emph{Conference on Neural Information Processing Systems
  (NeurIPS)}, 2018.

\bibitem[Huber(2011)]{huber2011robust}
Huber, P.~J.
\newblock \emph{Robust statistics}.
\newblock Springer, 2011.

\bibitem[Ioffe \& Szegedy(2015)Ioffe and Szegedy]{ioffe2015batch}
Ioffe, S. and Szegedy, C.
\newblock Batch normalization: Accelerating deep network training by reducing
  internal covariate shift.
\newblock In \emph{International Conference on Machine Learing (ICML)}, 2015.

\bibitem[Kajino et~al.(2012)Kajino, Tsuboi, and Kashima]{kajino2012convex}
Kajino, H., Tsuboi, Y., and Kashima, H.
\newblock A convex formulation for learning from crowds.
\newblock \emph{Transactions of the Japanese Society for Artificial
  Intelligence}, 27\penalty0 (3):\penalty0 133--142, 2012.

\bibitem[Mansour \& Schain(2014)Mansour and Schain]{mansour2014robust}
Mansour, Y. and Schain, M.
\newblock Robust domain adaptation.
\newblock \emph{Annals of Mathematics and Artificial Intelligence}, 71\penalty0
  (4):\penalty0 365--380, 2014.

\bibitem[Mansour et~al.(2009)Mansour, Mohri, and
  Rostamizadeh]{mansour2009domain}
Mansour, Y., Mohri, M., and Rostamizadeh, A.
\newblock Domain adaptation with multiple sources.
\newblock In \emph{Conference on Neural Information Processing Systems (NIPS)},
  pp.\  1041--1048, 2009.

\bibitem[McAuley et~al.(2015{\natexlab{a}})McAuley, Pandey, and
  Leskovec]{mcauley2015inferring}
McAuley, J., Pandey, R., and Leskovec, J.
\newblock Inferring networks of substitutable and complementary products.
\newblock In \emph{ACM SIGKDD International Conference on Knowledge Discovery
  and Data Mining}, 2015{\natexlab{a}}.

\bibitem[McAuley et~al.(2015{\natexlab{b}})McAuley, Targett, Shi, and Van
  Den~Hengel]{mcauley2015image}
McAuley, J., Targett, C., Shi, Q., and Van Den~Hengel, A.
\newblock Image-based recommendations on styles and substitutes.
\newblock In \emph{ACM SIGIR Conference on Research and Development in
  Information Retrieval}, 2015{\natexlab{b}}.

\bibitem[McKeen et~al.(2016)McKeen, Alexandrovich, Anati, Caspi, Johnson,
  Leslie-Hurd, and Rozas]{mckeen2016intel}
McKeen, F., Alexandrovich, I., Anati, I., Caspi, D., Johnson, S., Leslie-Hurd,
  R., and Rozas, C.
\newblock Intel{\textregistered} software guard extensions
  (intel{\textregistered} sgx) support for dynamic memory management inside an
  enclave.
\newblock In \emph{Proceedings of the Hardware and Architectural Support for
  Security and Privacy 2016}, pp.\ ~10, 2016.

\bibitem[McMahan et~al.(2017)McMahan, Moore, Ramage, Hampson, and
  y~Arcas]{mcmahan2017communication}
McMahan, B., Moore, E., Ramage, D., Hampson, S., and y~Arcas, B.~A.
\newblock Communication-efficient learning of deep networks from decentralized
  data.
\newblock In \emph{Conference on Uncertainty in Artificial Intelligence
  (AISTATS)}, 2017.

\bibitem[Mohri \& Medina(2012)Mohri and Medina]{mohri2012new}
Mohri, M. and Medina, A.~M.
\newblock New analysis and algorithm for learning with drifting distributions.
\newblock In \emph{International Conference on Algorithmic Learning Theory
  (ALT)}, 2012.

\bibitem[Pennington et~al.(2014)Pennington, Socher, and
  Manning]{pennington2014glove}
Pennington, J., Socher, R., and Manning, C.
\newblock Glove: Global vectors for word representation.
\newblock In \emph{Conference on Empirical Methods in Natural Language
  Processing (EMNLP)}, 2014.

\bibitem[Pentina \& Lampert(2017)Pentina and Lampert]{PenLam17}
Pentina, A. and Lampert, C.~H.
\newblock Multi-task learning with labeled and unlabeled tasks.
\newblock In \emph{International Conference on Machine Learing (ICML)}, 2017.

\bibitem[Prasad et~al.(2018)Prasad, Suggala, Balakrishnan, and
  Ravikumar]{prasad2018robust}
Prasad, A., Suggala, A.~S., Balakrishnan, S., and Ravikumar, P.
\newblock Robust estimation via robust gradient estimation.
\newblock \emph{arXiv preprint arXiv:1802.06485}, 2018.

\bibitem[Pregibon(1982)]{pregibon1982resistant}
Pregibon, D.
\newblock Resistant fits for some commonly used logistic models with medical
  applications.
\newblock \emph{Biometrics}, pp.\  485--498, 1982.

\bibitem[Qiao \& Valiant(2018)Qiao and Valiant]{qiao2018learning}
Qiao, M. and Valiant, G.
\newblock Learning discrete distributions from untrusted batches.
\newblock In \emph{LIPIcs-Leibniz International Proceedings in Informatics},
  volume~94. Schloss Dagstuhl-Leibniz-Zentrum fuer Informatik, 2018.

\bibitem[Russakovsky et~al.(2015)Russakovsky, Deng, Su, Krause, Satheesh, Ma,
  Huang, Karpathy, Khosla, Bernstein, et~al.]{russakovsky2015imagenet}
Russakovsky, O., Deng, J., Su, H., Krause, J., Satheesh, S., Ma, S., Huang, Z.,
  Karpathy, A., Khosla, A., Bernstein, M., et~al.
\newblock Imagenet large scale visual recognition challenge.
\newblock \emph{International Journal of Computer Vision (IJCV)}, 115\penalty0
  (3):\penalty0 211--252, 2015.

\bibitem[Shalev-Shwartz \& Ben-David(2014)Shalev-Shwartz and
  Ben-David]{shalev2014understanding}
Shalev-Shwartz, S. and Ben-David, S.
\newblock \emph{Understanding machine learning: From theory to algorithms}.
\newblock Cambridge University Press, 2014.

\bibitem[Shokri \& Shmatikov(2015)Shokri and Shmatikov]{shokri2015privacy}
Shokri, R. and Shmatikov, V.
\newblock Privacy-preserving deep learning.
\newblock In \emph{ACM SIGSAC conference on computer and communications
  security}, 2015.

\bibitem[Smith et~al.(2017)Smith, Chiang, Sanjabi, and
  Talwalkar]{smith2017federated}
Smith, V., Chiang, C.-K., Sanjabi, M., and Talwalkar, A.~S.
\newblock Federated multi-task learning.
\newblock In \emph{Conference on Neural Information Processing Systems (NIPS)},
  2017.

\bibitem[Song et~al.(2019)Song, He, Wang, and Hopcroft]{song2018improving}
Song, C., He, K., Wang, L., and Hopcroft, J.~E.
\newblock Improving the generalization of adversarial training with domain
  adaptation.
\newblock In \emph{International Conference on Learning Representations
  (ICLR)}, 2019.

\bibitem[Sun \& Lampert(2018)Sun and Lampert]{sun2018ksconf}
Sun, R. and Lampert, C.~H.
\newblock {KS(conf)}: A light-weight test if a convnet operates outside of its
  specifications.
\newblock In \emph{German Conference on Pattern Recognition (GCPR)}, 2018.

\bibitem[Tukey(1960)]{tukey1960survey}
Tukey, J.~W.
\newblock A survey of sampling from contaminated distributions.
\newblock \emph{Contributions to probability and statistics}, pp.\  448--485,
  1960.

\bibitem[Wahlsten et~al.(2003)Wahlsten, Metten, Phillips, Boehm,
  Burkhart-Kasch, Dorow, Doerksen, Downing, Fogarty, Rodd-Henricks,
  et~al.]{wahlsten2003different}
Wahlsten, D., Metten, P., Phillips, T.~J., Boehm, S.~L., Burkhart-Kasch, S.,
  Dorow, J., Doerksen, S., Downing, C., Fogarty, J., Rodd-Henricks, K., et~al.
\newblock Different data from different labs: lessons from studies of
  gene--environment interaction.
\newblock \emph{Journal of neurobiology}, 54\penalty0 (1):\penalty0 283--311,
  2003.

\bibitem[Wais et~al.(2010)Wais, Lingamneni, Cook, Fennell, Goldenberg, Lubarov,
  Marin, and Simons]{Wais10towardsbuilding}
Wais, P., Lingamneni, S., Cook, D., Fennell, J., Goldenberg, B., Lubarov, D.,
  Marin, D., and Simons, H.
\newblock Towards building a high-quality workforce with mechanical turk.
\newblock In \emph{NIPS Workshop on Computational Social Science and the Wisdom
  of Crowds}, 2010.

\bibitem[Xian et~al.(2018)Xian, Lampert, Schiele, and Akata]{xian2018zero}
Xian, Y., Lampert, C.~H., Schiele, B., and Akata, Z.
\newblock Zero-shot learning-a comprehensive evaluation of the good, the bad
  and the ugly.
\newblock \emph{IEEE Transactions on Pattern Analysis and Machine Intelligence
  (T-PAMI)}, 2018.

\bibitem[Xie(2017)]{xie2017robust}
Xie, T.
\newblock \emph{Robust Learning from Multiple Information Sources}.
\newblock PhD thesis, University of Michigan, 2017.

\bibitem[Yin et~al.(2018)Yin, Chen, Kannan, and Bartlett]{pmlr-v80-yin18a}
Yin, D., Chen, Y., Kannan, R., and Bartlett, P.
\newblock {B}yzantine-robust distributed learning: Towards optimal statistical
  rates.
\newblock In \emph{International Conference on Machine Learing (ICML)}, 2018.

\bibitem[Zhang et~al.(2012)Zhang, Zhang, and Ye]{zhang2012generalizationNIPS}
Zhang, C., Zhang, L., and Ye, J.
\newblock Generalization bounds for domain adaptation.
\newblock In \emph{Conference on Neural Information Processing Systems (NIPS)},
  2012.

\bibitem[Zhang et~al.(2013)Zhang, Zhang, and Ye]{zhang2012generalizationArxiv}
Zhang, C., Zhang, L., and Ye, J.
\newblock Generalization bounds for domain adaptation.
\newblock \emph{arXiv preprint arXiv:1304.1574}, 2013.

\bibitem[Zhang et~al.(2017)Zhang, Iwata, and Kashima]{zhang2017robust}
Zhang, G., Iwata, T., and Kashima, H.
\newblock Robust multi-view topic modeling by incorporating detecting
  anomalies.
\newblock In \emph{Joint European Conference on Machine Learning and Knowledge
  Discovery in Databases}, 2017.

\bibitem[Zhao et~al.(2017)Zhao, Xie, Xu, and Sun]{zhao2017multi}
Zhao, J., Xie, X., Xu, X., and Sun, S.
\newblock Multi-view learning overview: Recent progress and new challenges.
\newblock \emph{Information Fusion}, 38:\penalty0 43--54, 2017.

\bibitem[Zimin \& Lampert(2017)Zimin and Lampert]{zimin2017learning}
Zimin, A. and Lampert, C.~H.
\newblock Learning theory for conditional risk minimization.
\newblock In \emph{Conference on Uncertainty in Artificial Intelligence
  (AISTATS)}, 2017.

\end{thebibliography}
\bibliographystyle{icml2019}

\appendix

\addtolength{\abovedisplayskip}{.15\baselineskip}
\addtolength{\belowdisplayskip}{.15\baselineskip}
\addtolength{\abovedisplayshortskip}{1.0\baselineskip}
\addtolength{\belowdisplayshortskip}{1.0\baselineskip}

\clearpage
\section{Proof of Theorem \ref{thm:main_bound}}
\label{app:thm_proof}
First we bound $|\hat{\epsilon}_{\alpha}\left(h\right) - \epsilon_{\alpha}\left(h\right)|$ with high probability and uniformly over $\mathcal{H}$. We adapt the classical proofs of generalization bounds in terms of the Rademacher complexity of a hypothesis class, \eg \cite{bousquet2004introduction}.
\begin{proposition}
\label{prop:uniform_bound}
Given the setup and assumptions described above, for any $\delta > 0$ with probability at least $1 - \delta$ over the data, for any function $h\in\mathcal{H}$:
\begin{equation}
\label{eqn:empirical_vs_true_mean}
\begin{split}
    |\epsilon_{\alpha}\left(h\right) - \hat{\epsilon}_{\alpha} \left(h\right)| & \leq 2\sum_{i=1}^N \alpha_i \mathcal{R}_i \left(\mathcal{H}\right) \\ & + 3 \sqrt{\frac{\log\left(\frac{4}{\delta}\right)M^2}{2}}\sqrt{\sum_{i=1}^N\frac{\alpha_i^2}{m_i}},
\end{split}
\end{equation}
where for each $i = 1, 2, \ldots, N$:
\begin{equation}
\label{eqn:define_rademacher_i}
\mathcal{R}_i \left(\mathcal{H}\right) = \mathbb{E}_{\sigma}\left(\sup_{f\in\mathcal{H}}\left(\frac{1}{m_i}\sum_{j=1}^{m_i}\sigma_{i,j}L(f(x_{i,j}), y_{i,j})\right)\right),
\end{equation}
and where $\sigma_{i,j}$ are independent Rademacher random variables.
\end{proposition}
\begin{proof}
Write:
\begin{equation}
\label{eqn:equation_i_need}
    \epsilon_{\alpha}\left(h\right) \leq \hat{\epsilon}_{\alpha}\left(h\right) + \sup_{f\in \mathcal{H}}\left(\epsilon_{\alpha}\left(f\right)- \hat{\epsilon}_{\alpha}\left(f\right)\right)
\end{equation}
To link the second term to its expectation, we prove the following:
\begin{lemma}
\label{lem:mcdiar_condition}
Define the function $\phi:\left(\mathcal{X}\times\mathcal{Y}\right)^m \rightarrow \mathbb{R}$ by:
\begin{small}
$$ \phi\left(\{x_{1,1}, y_{1,1}\}, \ldots, \{x_{N, m_N}, y_{N, m_N}\}\right) = \sup_{f\in \mathcal{H}}\left(\epsilon_{\alpha}\left(f\right)- \hat{\epsilon}_{\alpha}\left(f\right)\right).$$
\end{small}
Denote for brevity $z_{i,j} = \{x_{i,j}, y_{i,j}\}$. Then, for any $i \in \{1, 2, \ldots, N\}, j \in \{1, 2, \ldots, m_i\}$:
\begin{equation}
\begin{split}
    \sup_{z_{1,1}, \ldots, z_{N, m_N}, z_{i,j}^{'}} & |\phi\left(z_{1,1}, \ldots, z_{i,j}, \ldots, z_{N, m_N}\right) \\ & - \phi\left(z_{1,1}, \ldots, z_{i,j}^{'}, \ldots, z_{N, m_N}\right)| \leq \frac{\alpha_i}{m_i}M
\end{split}
\end{equation}
\end{lemma}
\begin{proof}
Fix any $i, j$ and any $z_{1,1}, \ldots, z_{N, m_N}, z_{i,j}^{'}$. Denote the $\alpha$-weighted empirical average of the loss with respect to the sample $z_{1,1}, \ldots, z_{i,j}^{'}, \ldots, z_{N, m_N}$ by $\epsilon_{\alpha}^{'}$. Then we have that:
\begin{align*}
    |\phi\left(\ldots, z_{i,j}, \ldots\right) & - \phi (\ldots, z_{i,j}^{'}, \ldots )| \\ & =  |\sup_{f\in\mathcal{H}}\left(\epsilon_{\alpha}\left(f\right) - \hat{\epsilon}_{\alpha}\left(f\right)\right) \\ & - \sup_{f\in\mathcal{H}} (\epsilon_{\alpha}\left(f\right) - \hat{\epsilon}_{\alpha}^{'}\left(f\right) )| \\ & \leq |\sup_{f\in\mathcal{H}}(\hat{\epsilon}^{'}_{\alpha}\left(f\right) - \hat{\epsilon}_{\alpha}\left(f\right))| \\ & = \frac{\alpha_i}{m_i}|\sup_{f\in\mathcal{H}}\left(L\left(f(x_{i,j}^{'}), y_{i,j}^{'}\right) - L\left(f(x_{i,j}), y_{i,j}\right)\right)| \\ & \leq \frac{\alpha_i}{m_i}M
\end{align*}
Note: the inequality we used above holds for bounded functions inside the supremum.
\end{proof}
\noindent Let $S$ denote a random sample of size $m$ drawn from a distribution as the one generating out data (\ie $m_i$ samples from $\mathcal{D}_i$ for each $i$). Now, using Lemma \ref{lem:mcdiar_condition}, McDiarmid's inequality gives:
\begin{equation*}
\begin{split}
    \mathbb{P}\left(\phi(S) - \mathbb{E}(\phi(S)) \geq t\right) & \leq \exp\left(-\frac{2t^2}{\sum_{i=1}^N\sum_{j=1}^{m_i}\frac{\alpha_i^2}{m_i^2}M^2} \right) \\ & = \exp\left(-\frac{2t^2}{M^2\sum_{i=1}^N \frac{\alpha_i^2}{m_i}}\right)
\end{split}
\end{equation*}
For any $\delta > 0$, setting the right-hand side above to be $\delta/4$ and using (\ref{eqn:equation_i_need}), we obtain that with probability at least $1-\delta/4$:
\begin{equation}
\begin{split}
    \epsilon_{\alpha}\left(h\right) \leq \hat{\epsilon}_{\alpha}\left(h\right) & + \mathbb{E}_S\left(\sup_{f\in\mathcal{H}}\left(\epsilon_{\alpha} (f) - \hat{\epsilon}_{\alpha}(f)\right)\right) \\ & + \sqrt{\frac{\log\left(\frac{4}{\delta}\right)M^2}{2}}\sqrt{\sum_{i=1}^N\frac{\alpha_i^2}{m_i}}
\end{split}
\end{equation}
To deal with the expected loss inside the second term, introduce a ghost sample (denoted by $S'$), drawn from the same distributions as our original sample (denoted by $S$). Denoting the weighted empirical loss with respect to the ghost sample by $\epsilon_{\alpha}^{'}$, $\beta_i = m_i/m$ for all $i$, and using the convexity of the supremum, we obtain:
\begin{equation*}
\begin{split}
    \mathbb{E}_S & \left(\sup_{f\in\mathcal{H}}\left(\epsilon_{\alpha} (f) - \hat{\epsilon}_{\alpha}(f)\right)\right) \\ & = \mathbb{E}_{S}\left(\sup_{f\in\mathcal{H}}\left(\mathbb{E}_{S'}\left(\hat{\epsilon}_{\alpha}^{'}(f)\right) - \hat{\epsilon}_{\alpha}(f)\right)\right) \\ & \leq \mathbb{E}_{S, S'} \left(\sup_{f\in\mathcal{H}}\left(\hat{\epsilon}_{\alpha}^{'}(f) - \hat{\epsilon}_{\alpha}(f) \right)\right) \\ & = \mathbb{E}_{S, S'}\left(\sup_{f\in\mathcal{H}}\left(\frac{1}{m}\sum_{i=1}^N\sum_{j=1}^{m_i}\frac{\alpha_i}{\beta_i}\left(L(f(x_{i,j}^{'}), y_{i,j}^{'}) \right. \right. \right. \\ &  \left. \left. \left. \quad - L(f(x_{i,j}), y_{i,j}) \vphantom{L^{'}}\right) \vphantom{\frac{1}{m}\sum_{i=1}^N}\right)\right)
\end{split}
\end{equation*}

Introducing $m$ independent Rademacher random variables and noting that $L(f(x^{'}), y^{'}) - L(f(x), y)$ and $\sigma\left(L(f(x^{'}), y^{'}) - L(f(x), y)\right)$ have the same distribution, as long as $\left(x, y\right)$ and $(x^{'},y^{'})$ have the same distribution:
\begin{equation*}
\begin{split}
    \mathbb{E}_S & \left(\sup_{f\in\mathcal{H}}\left(\epsilon_{\alpha} (f) - \hat{\epsilon}_{\alpha}(f)\right)\right) \\ & \leq \mathbb{E}_{S, S', \sigma}\left(\sup_{f\in\mathcal{H}}\left(\frac{1}{m}\sum_{i=1}^N\sum_{j=1}^{m_i}\frac{\alpha_i}{\beta_i}\sigma_{i,j}\left(L(f(x_{i,j}^{'}), y_{i,j}^{'}) \right. \right. \right. \\ & \left. \left. \left. \quad - L(f(x_{i,j}), y_{i,j}) \vphantom{L^{'}}\right) \vphantom{\frac{1}{m}\sum_{i=1}^N}\right)\right) \\ & \leq \mathbb{E}_{S^{'}, \sigma}\left(\sup_{f\in\mathcal{H}}\left(\frac{1}{m}\sum_{i=1}^N\sum_{j=1}^{m_i}\frac{\alpha_{i}}{\beta_{i}}\sigma_{i,j}L(f(x_{i,j}^{'}), y^{'}_{i,j})\right)\right) \\ & + \mathbb{E}_{S, \sigma}\left(\sup_{f\in\mathcal{H}}\left(\frac{1}{m}\sum_{i=1}^N\sum_{j=1}^{m_i}\frac{\alpha_{i}}{\beta_{i}}\left(-\sigma_{i,j}\right)L(f(x_{i,j}), y_{i,j})\right)\right) \\ & = 2\mathbb{E}_{S, \sigma}\left(\sup_{f\in\mathcal{H}}\left(\frac{1}{m}\sum_{i=1}^N\sum_{j=1}^{m_i}\frac{\alpha_{i}}{\beta_{i}}\sigma_{i,j}L(f(x_{i,j}), y_{i,j})\right)\right)
\end{split}
\end{equation*}
We can now link the last term to the empirical analog of the Rademacher complexity, by using the McDiarmid Inequality (with an observation similar to Lemma 1). Putting this together, we obtain that for any $\delta > 0$ with probability at least $1 - \delta/2$:
\begin{equation}
\begin{split}
    \epsilon_{\alpha}\left(h\right) & \leq \hat{\epsilon}_{\alpha} \left(h\right) \\ & + 2\mathbb{E}_{\sigma}\left(\sup_{f\in\mathcal{H}}\left(\frac{1}{m}\sum_{i=1}^N\sum_{j=1}^{m_i}\frac{\alpha_{i}}{\beta_{i}}\sigma_{i,j}L(f(x_{i,j}), y_{i,j})\right)\right) \\ & + 3 \sqrt{\frac{\log\left(\frac{4}{\delta}\right)M^2}{2}}\sqrt{\sum_{i=1}^N\frac{\alpha_i^2}{m_i}}
\end{split}
\end{equation}
Finally, note that:
\begin{align*}
    \mathbb{E}_{\sigma} & \left(\sup_{f\in\mathcal{H}}\left(\frac{1}{m}\sum_{i=1}^N\sum_{j=1}^{m_i}\frac{\alpha_{i}}{\beta_{i}}\sigma_{i,j}L(f(x_{i,j}), y_{i,j})\right)\right) \\ & \leq \mathbb{E}_{\sigma}\left(\sum_{i=1}^{N}\alpha_i\sup_{f\in\mathcal{H}}\left(\frac{1}{m_i}\sum_{j=1}^{m_i}\sigma_{i,j}L(f(x_{i,j}), y_{i,j})\right)\right) \\ & = \sum_{i=1}^N \alpha_i \mathbb{E}_{\sigma}\left(\sup_{f\in\mathcal{H}}\left(\frac{1}{m_i}\sum_{j=1}^{m_i}\sigma_{i,j}L(f(x_{i,j}), y_{i,j})\right)\right) \\ & = \sum_{i=1}^N \alpha_i \mathcal{R}_i \left(\mathcal{H}\right)
\end{align*}
Bounding $\hat{\epsilon}_{\alpha}(h) - \epsilon_{\alpha}(h)$ with the same quantity and with probability at least $1 - \delta/2$ follows by a similar argument. The result then follows by applying the union bound.
\end{proof}

\noindent Now we show:
\mainthm*
\begin{proof}
For any $h\in\mathcal{H}$:
\begin{align*}
    \lvert\epsilon_{\alpha}(h) - \epsilon_{T}(h)\rvert & = \lvert\sum_{i=1}^N \alpha_i \epsilon_i (h) - \epsilon_T (h)\rvert \\ & \leq \sum_{i=1}^N \alpha_i\lvert\epsilon_i (h) - \epsilon_T (h)\rvert \\ & \leq \sum_{i=1}^N \alpha_i d_{\mathcal{H}}\left(\mathcal{D}_i, \mathcal{D}_T\right).
\end{align*}
Now applying this bound twice and using Proposition \ref{prop:uniform_bound}, we get that with probability at least $1-\delta$:
\begin{align*}
    \epsilon_T(\hat{h}_{\alpha}) & \leq \epsilon_{\alpha}(\hat{h}_{\alpha}) + \sum_{i=1}^N \alpha_i d_{\mathcal{H}}\left(\mathcal{D}_i, \mathcal{D}_T\right) \\ & \leq \hat{\epsilon}_{\alpha}(\hat{h}_{\alpha}) + 2\sum_{i=1}^N \alpha_i \mathcal{R}_i \left(\mathcal{H}\right) \\ & + 3 \sqrt{\frac{\log\left(\frac{4}{\delta}\right)M^2}{2}}\sqrt{\sum_{i=1}^N\frac{\alpha_i^2}{m_i}} + \sum_{i=1}^N \alpha_i d_{\mathcal{H}}\left(\mathcal{D}_i, \mathcal{D}_T\right) \\ & \leq \hat{\epsilon}_{\alpha}(h_{T}^{*}) + 2\sum_{i=1}^N \alpha_i \mathcal{R}_i \left(\mathcal{H}\right) \\ & + 3 \sqrt{\frac{\log\left(\frac{4}{\delta}\right)M^2}{2}}\sqrt{\sum_{i=1}^N\frac{\alpha_i^2}{m_i}} + \sum_{i=1}^N \alpha_i d_{\mathcal{H}}\left(\mathcal{D}_i, \mathcal{D}_T\right) \\ & \leq \epsilon_{\alpha}(h_T^*) + 4\sum_{i=1}^N \alpha_i \mathcal{R}_i \left(\mathcal{H}\right) \\ & + 6 \sqrt{\frac{\log\left(\frac{4}{\delta}\right)M^2}{2}}\sqrt{\sum_{i=1}^N\frac{\alpha_i^2}{m_i}} + \sum_{i=1}^N \alpha_i d_{\mathcal{H}}\left(\mathcal{D}_i, \mathcal{D}_T\right) \\ & \leq \epsilon_T (h_T^{*}) + 4\sum_{i=1}^N \alpha_i \mathcal{R}_i \left(\mathcal{H}\right) \\ & + 6 \sqrt{\frac{\log\left(\frac{4}{\delta}\right)M^2}{2}}\sqrt{\sum_{i=1}^N\frac{\alpha_i^2}{m_i}} + 2\sum_{i=1}^N \alpha_i d_{\mathcal{H}}\left(\mathcal{D}_i, \mathcal{D}_T\right)
\end{align*}
\end{proof}

\section{Details about Algorithm \ref{alg:main_algo}}
\label{app:details}
\subsection{Distribution-independent upper bounds on the Rademacher complexity}

Here we give examples of some well-known upper bounds on the Rademacher complexity of certain function classes, which are distribution-independent. Applying such a bound on the Rademacher terms in Theorem \ref{thm:main_bound} will make the dependence of the second term in the bound on the weights disappear. Therefore, we focus on the remaining terms in our algorithm. 

Throughout this section, we discuss the Rademacher complexity of a function class $\mathcal{H}$ with respect to a set of samples $\{x_1, \ldots, x_n\} \sim \mathcal{D}$, defined as:
\begin{equation}
\label{eqn:general_rademacher}
\mathcal{R}\left(\mathcal{H}\right) = \mathbb{E}_{\sigma}\left(\sup_{f\in\mathcal{H}}\left(\frac{1}{n}\sum_{i=1}^{n}\sigma_{i}f(x_{i})\right)\right) 
\end{equation}

In the case of bounded binary linear classifiers $\mathcal{H} = \{x\rightarrow\langle \textbf{w},\textbf{x}\rangle: \|\textbf{w}\|_2 \leq B\}$, acting on a bounded domain $\mathcal{X}$ (\ie for all $x\in\mathcal{X}, \|x\|_2 \leq D$), Lemma 26.10 in \cite{shalev2014understanding} shows that:
$$\mathcal{R}\left(\mathcal{H}\right) \leq \frac{BD}{\sqrt{n}}.$$
More generally, the Rademacher complexity of a set of binary classifiers with a finite VC dimention $h$ can be bounded by \cite{bousquet2004introduction}:
$$\mathcal{R}\left(\mathcal{H}\right) \leq C\sqrt{\frac{h}{n}},$$
for some constant $C$. The Rademacher complexity is also related to another popular complexity measure, the covering number, via Dudley's entropy bound \cite{bousquet2004introduction}:
$$\mathcal{R}\left(\mathcal{H}\right) \leq \frac{C}{\sqrt{n}}\int_{0}^{\infty}\sqrt{\log N\left(\mathcal{H}, t, n\right)} \, dt,$$
where $N\left(\mathcal{H}, t, n\right)$ is the size of the smallest $t$-cover of the space $\mathcal{H}$, under the metric: $$d_n\left(h, h'\right) = \frac{1}{n}|\{h(x_i) \neq h'(x_i): i = 1, \ldots, n\}|.$$
Note that both the VC-dimension and the covering number are distribution-independent measures of complexity, so the corresponding upper bounds do not depend on $\mathcal{D}$ as well.
\subsection{Computing the empirical discrepancies}
As explained in Section \ref{sec:algorithm}, the discrepancies:
\begin{align}
d_{\mathcal{H}}\left(\mathcal{D}_i, \mathcal{D}_T\right) = \sup_{h\in\mathcal{H}}\left(|\epsilon_i (h) - \epsilon_T (h)|\right).
\end{align}
are unknown in practice and therefore need to be estimated from their empirical counterparts:
\begin{equation}
\begin{split}
d_{\mathcal{H}}\left(S_i, S_T\right) & = \sup_{h\in\mathcal{H}}\left(|\hat{\epsilon}_i (h) - \hat{\epsilon}_T (h)|\right) \\ & = \sup_{h\in\mathcal{H}}(|\frac{1}{m_i}\sum_{j=1}^{m_i} L\left(h\left(x_{i,j}\right), y_{i,j}\right) \\ & \quad \quad \text{  } - \frac{1}{m_T} \sum_{j=1}^{m_T} L\left(h\left(x_{T, j}\right), y_{T, j}\right)|).
\end{split}
\end{equation} 
Here, we explain how these are computed in our experiments. Notice that for the $0/1$-loss, a symmetric hypothesis class $\mathcal{H}$ and whenever $\mathcal{Y} = \{-1, +1\}$, we have:
\begin{small}
\begin{equation}
\label{eqn:flipped_labels}
\begin{split}
& d_{\mathcal{H}}\left(S_i, S_T\right) \\ & = \sup_{h\in\mathcal{H}}|\frac{1}{m_i}\sum_{j=1}^{m_i} \mathds{1}_{\{h(x_{i,j})y_{i,j}<0\}} - \frac{1}{m_T} \sum_{j=1}^{m_T} \mathds{1}_{\{h(x_{T,j})y_{T,j}<0\}}| \\ & = \sup_{h\in\mathcal{H}}\left(\frac{1}{m_i}\sum_{j=1}^{m_i} \mathds{1}_{\{h(x_{i,j})y_{i,j}<0\}} - \frac{1}{m_T} \sum_{j=1}^{m_T} \mathds{1}_{\{h(x_{T,j})y_{T,j}<0\}}\right) \\ & =  \sup_{h\in\mathcal{H}}\left(1 - (\frac{1}{m_i}\sum_{j=1}^{m_i} \mathds{1}_{\{h(x_{i,j})\bar{y}_{i,j}<0\}} + \frac{1}{m_T}\sum_{j=1}^{m_T} \mathds{1}_{\{h(x_{T,j})y_{T,j}<0\}})\right) \\ & = 1 - \inf_{h\in\mathcal{H}} \left(\frac{1}{m_i}\sum_{j=1}^{m_i} \mathds{1}_{\{h(x_{i,j})\bar{y}_{i,j}<0\}} + \frac{1}{m_T}\sum_{j=1}^{m_T} \mathds{1}_{\{h(x_{T,j})y_{T,j}<0\}}\right),
\end{split}
\end{equation}
\end{small} 
where $\bar{y}_{i,j} = 1 - y_{i,j}$ is the flipped label of the $j$-th data point from the $i$-th source. Now notice that computing the infimum in equation (\ref{eqn:flipped_labels}) is equivalent to solving a (weighted) empirical risk minimization problem with the input data from the source and the target merged and the labels being the flipped labels from the source and the actual labels from the target. 

Therefore, computing the empirical discrepancies is equivalent to solving an empirical risk minimization problem and standard convex upper bounds can be applied to make the problem tractable. In our experiments, we solve the ERM problem by using square loss.

\clearpage
\section{Additional results from experiments}
\label{app:experiments}
Over the next pages we present more detailed results from the experiments on the Animals with Attributes 2 dataset. The table from the main body of the paper (Table 1) is split according to the type of data corruption. In addition, we performed experiments in which a proportion $p$ of the samples in the $n$ corrupted sources are modified (instead of all of them). Apart from $p = 1$, we experimented with $p = 0.5$ and $p = 0.2$. We present the same type of results for these cases, together with a more detailed breakdown, depending on the type of corruption.

\begin{table*}
\centering
\caption{Summary of the results for $p = 1$, over all 85 prediction tasks and all corruptions (same as Table 1).}
 \label{table:animals_p_large_appendix}
 \begin{tabular}{|c || c || c | c | c | c | c | c || c |} 
 \hline
 \backslashbox{Baseline}{$n$} & $n = 0$ & $n = 10$ & $n = 20$ & $n = 30$ & $n = 40$ & $n = 50$ & $n = 55$ & $n = 59$ \\ 
 \hline\hline 
Reference only & 84/1/0 & 505/5/0 & 497/13/0 & 487/23/0 & 475/35/0 & 442/68/0 & 325/185/0 & 0/510/0 \\
All data & 0/85/0 & 115/395/0 & 267/243/0 & 370/140/0 & 438/72/0 & 468/42/0 & 479/31/0 & 484/26/0 \\
Median of probs. & 9/76/0 & 47/463/0 & 172/338/0 & 336/174/0 & 469/41/0 & 504/6/0 & 502/8/0 & 499/11/0 \\
\cite{feng2014distributed} & 8/77/0 & 32/478/0 & 110/400/0 & 338/172/0 & 457/53/0 & 504/6/0 & 502/8/0 & 497/13/0 \\
\cite{pmlr-v80-yin18a} & 14/71/0 & 179/331/0 & 390/120/0 & 432/78/0 & 472/38/0 & 502/8/0 & 503/7/0 & 497/13/0 \\
\cite{pregibon1982resistant} & 55/30/0 & 308/202/0 & 361/149/0 & 416/94/0 & 437/73/0 & 455/55/0 & 470/40/0 & 485/25/0 \\
Batch norm & 0/85/0 & 107/403/0 & 317/193/0 & 416/94/0 & 446/63/1 & 478/32/0 & 487/23/0 & 482/28/0 \\
 \hline
 \end{tabular}
\end{table*}

\begin{table*}
\caption{Summary of results for $p = 1$, split by the type of data corruption}
\begin{subtable}{1\textwidth}
\centering
\caption{Summary of the results for $p = 1$ and label bias, over all 85 prediction tasks}
 \label{table:animals_p_large_enforce_label}
 \begin{tabular}{|c || c | c | c | c | c | c || c |} 
 \hline
 \backslashbox{Baseline}{$n$} & $n = 10$ & $n = 20$ & $n = 30$ & $n = 40$ & $n = 50$ & $n = 55$ & $n = 59$ \\ 
 \hline\hline 
Reference only & 85/0/0 & 84/1/0 & 82/3/0 & 80/5/0 & 76/9/0 & 55/30/0 & 0/85/0 \\
All data & 61/24/0 & 82/3/0 & 85/0/0 & 85/0/0 & 85/0/0 & 85/0/0 & 84/1/0 \\
Median of probs. & 23/62/0 & 81/4/0 & 85/0/0 & 85/0/0 & 85/0/0 & 85/0/0 & 84/1/0 \\
\cite{feng2014distributed} & 4/81/0 & 19/66/0 & 85/0/0 & 85/0/0 & 85/0/0 & 85/0/0 & 84/1/0 \\
\cite{pmlr-v80-yin18a} & 50/35/0 & 85/0/0 & 84/1/0 & 85/0/0 & 85/0/0 & 85/0/0 & 84/1/0 \\
\cite{pregibon1982resistant} & 51/34/0 & 64/21/0 & 84/1/0 & 84/1/0 & 84/1/0 & 83/2/0 & 83/2/0 \\
Batch norm & 53/32/0 & 81/4/0 & 85/0/0 & 85/0/0 & 85/0/0 & 85/0/0 & 84/1/0 \\
 \hline
 \end{tabular}
 
\end{subtable}
\hfill
\begin{subtable}{1\textwidth}
\centering
 \caption{Summary of the results for $p = 1$ and shuffled labels, over all 85 prediction tasks}
 \label{table:animals_p_large_shuffle_labels}
 \begin{tabular}{|c || c | c | c | c | c | c || c |} 
 \hline
 \backslashbox{Baseline}{$n$} & $n = 10$ & $n = 20$ & $n = 30$ & $n = 40$ & $n = 50$ & $n = 55$ & $n = 59$ \\ 
 \hline\hline 
Reference only & 83/2/0 & 77/8/0 & 71/14/0 & 65/20/0 & 56/29/0 & 38/47/0 & 0/85/0 \\
All data & 4/81/0 & 51/34/0 & 69/16/0 & 74/11/0 & 82/3/0 & 79/6/0 & 79/6/0 \\
Median of probs. & 2/83/0 & 24/61/0 & 63/22/0 & 83/2/0 & 80/5/0 & 79/6/0 & 80/5/0 \\
\cite{feng2014distributed} & 3/82/0 & 0/85/0 & 51/34/0 & 77/8/0 & 80/5/0 & 79/6/0 & 80/5/0 \\
\cite{pmlr-v80-yin18a} & 32/53/0 & 63/22/0 & 62/23/0 & 76/9/0 & 80/5/0 & 79/6/0 & 80/5/0 \\
\cite{pregibon1982resistant} & 57/28/0 & 63/22/0 & 70/15/0 & 71/14/0 & 75/10/0 & 73/12/0 & 71/14/0 \\
Batch norm & 3/82/0 & 41/44/0 & 60/25/0 & 63/21/1 & 79/6/0 & 79/6/0 & 79/6/0 \\
 \hline
 \end{tabular}
\end{subtable}
\hfill
\begin{subtable}{1\textwidth}
\centering
 \caption{Summary of the results for $p = 1$ and shuffled features, over all 85 prediction tasks}
 \label{table:animals_p_large_shuffle_features}
 \begin{tabular}{|c || c | c | c | c | c | c || c |} 
 \hline
 \backslashbox{Baseline}{$n$} & $n = 10$ & $n = 20$ & $n = 30$ & $n = 40$ & $n = 50$ & $n = 55$ & $n = 59$ \\ 
 \hline\hline 
Reference only & 84/1/0 & 85/0/0 & 83/2/0 & 81/4/0 & 78/7/0 & 62/23/0 & 0/85/0 \\
All data & 39/46/0 & 60/25/0 & 68/17/0 & 73/12/0 & 77/8/0 & 80/5/0 & 85/0/0 \\
Median of probs. & 6/79/0 & 30/55/0 & 84/1/0 & 85/0/0 & 85/0/0 & 85/0/0 & 85/0/0 \\
\cite{feng2014distributed} & 10/75/0 & 72/13/0 & 84/1/0 & 85/0/0 & 85/0/0 & 85/0/0 & 85/0/0 \\
\cite{pmlr-v80-yin18a} & 18/67/0 & 76/9/0 & 84/1/0 & 85/0/0 & 85/0/0 & 85/0/0 & 85/0/0 \\
\cite{pregibon1982resistant} & 56/29/0 & 58/27/0 & 67/18/0 & 74/11/0 & 78/7/0 & 85/0/0 & 85/0/0 \\
Batch norm & 40/45/0 & 66/19/0 & 72/13/0 & 77/8/0 & 78/7/0 & 79/6/0 & 80/5/0 \\
 \hline
 \end{tabular}
\end{subtable}
\end{table*}

\pagebreak

\begin{table*}
\ContinuedFloat
\begin{subtable}{1\textwidth}
\centering
 \caption{Summary of the results for $p = 1$ and blurred images, over all 85 prediction tasks}
 \label{table:animals_p_large_blured}
 \begin{tabular}{|c || c | c | c | c | c | c || c |} 
 \hline
 \backslashbox{Baseline}{$n$} & $n = 10$ & $n = 20$ & $n = 30$ & $n = 40$ & $n = 50$ & $n = 55$ & $n = 59$ \\ 
 \hline\hline 
Reference only & 84/1/0 & 84/1/0 & 83/2/0 & 82/3/0 & 77/8/0 & 59/26/0 & 0/85/0 \\
All data & 0/85/0 & 2/83/0 & 26/59/0 & 53/32/0 & 66/19/0 & 75/10/0 & 78/7/0 \\
Median of probs. & 5/80/0 & 6/79/0 & 19/66/0 & 72/13/0 & 85/0/0 & 85/0/0 & 85/0/0 \\
\cite{feng2014distributed} & 5/80/0 & 6/79/0 & 30/55/0 & 70/15/0 & 85/0/0 & 85/0/0 & 84/1/0 \\
\cite{pmlr-v80-yin18a} & 26/59/0 & 47/38/0 & 61/24/0 & 73/12/0 & 84/1/0 & 85/0/0 & 84/1/0 \\
\cite{pregibon1982resistant} & 61/24/0 & 71/14/0 & 75/10/0 & 81/4/0 & 83/2/0 & 85/0/0 & 85/0/0 \\
Batch norm & 8/77/0 & 44/41/0 & 65/20/0 & 72/13/0 & 78/7/0 & 82/3/0 & 81/4/0 \\
 \hline
 \end{tabular}
\end{subtable}
\hfill
\begin{subtable}{1\textwidth}
\centering
 \caption{Summary of the results for $p = 1$ and dead pixels, over all 85 prediction tasks}
 \label{table:animals_p_large_pixels}
 \begin{tabular}{|c || c | c | c | c | c | c || c |} 
 \hline
 \backslashbox{Baseline}{$n$} & $n = 10$ & $n = 20$ & $n = 30$ & $n = 40$ & $n = 50$ & $n = 55$ & $n = 59$ \\ 
 \hline\hline 
Reference only & 85/0/0 & 83/2/0 & 84/1/0 & 84/1/0 & 77/8/0 & 58/27/0 & 0/85/0 \\
All data & 0/85/0 & 12/73/0 & 44/41/0 & 70/15/0 & 74/11/0 & 77/8/0 & 78/7/0 \\
Median of probs. & 6/79/0 & 6/79/0 & 14/71/0 & 61/24/0 & 85/0/0 & 85/0/0 & 85/0/0 \\
\cite{feng2014distributed} & 6/79/0 & 6/79/0 & 25/60/0 & 58/27/0 & 85/0/0 & 85/0/0 & 84/1/0 \\
\cite{pmlr-v80-yin18a} & 23/62/0 & 51/34/0 & 68/17/0 & 70/15/0 & 84/1/0 & 85/0/0 & 84/1/0 \\
\cite{pregibon1982resistant} & 28/57/0 & 38/47/0 & 51/34/0 & 52/33/0 & 56/29/0 & 69/16/0 & 85/0/0 \\
Batch norm & 1/84/0 & 28/57/0 & 59/26/0 & 69/16/0 & 74/11/0 & 79/6/0 & 79/6/0 \\
 \hline
 \end{tabular}
\end{subtable}
\hfill
\begin{subtable}{1\textwidth}
\centering
 \caption{Summary of the results for $p = 1$ and RGB channels swapped, over all 85 prediction tasks}
 \label{table:animals_p_large_RGB}
 \begin{tabular}{|c || c | c | c | c | c | c || c |} 
 \hline
 \backslashbox{Baseline}{$n$} & $n = 10$ & $n = 20$ & $n = 30$ & $n = 40$ & $n = 50$ & $n = 55$ & $n = 59$ \\ 
 \hline\hline 
Reference only & 84/1/0 & 84/1/0 & 84/1/0 & 83/2/0 & 78/7/0 & 53/32/0 & 0/85/0 \\
All data & 11/74/0 & 60/25/0 & 78/7/0 & 83/2/0 & 84/1/0 & 83/2/0 & 80/5/0 \\
Median of probs. & 5/80/0 & 25/60/0 & 71/14/0 & 83/2/0 & 84/1/0 & 83/2/0 & 80/5/0 \\
\cite{feng2014distributed} & 4/81/0 & 7/78/0 & 63/22/0 & 82/3/0 & 84/1/0 & 83/2/0 & 80/5/0 \\
\cite{pmlr-v80-yin18a} & 30/55/0 & 68/17/0 & 73/12/0 & 83/2/0 & 84/1/0 & 84/1/0 & 80/5/0 \\
\cite{pregibon1982resistant} & 55/30/0 & 67/18/0 & 69/16/0 & 75/10/0 & 79/6/0 & 75/10/0 & 76/9/0 \\
Batch norm & 2/83/0 & 57/28/0 & 75/10/0 & 80/5/0 & 84/1/0 & 83/2/0 & 79/6/0 \\
 \hline
 \end{tabular}
\end{subtable}
\end{table*}


\begin{table*}[h!]
\centering
\caption{Summary of the results for $p = 0.5$, over all 85 prediction tasks and all corruptions.}
 \label{table:animals_p_half_appendix}
 \begin{tabular}{|c || c || c | c | c | c | c | c || c |} 
 \hline
 \backslashbox{Baseline}{$n$} & $n = 0$ & $n = 10$ & $n = 20$ & $n = 30$ & $n = 40$ & $n = 50$ & $n = 55$ & $n = 59$ \\ 
 \hline\hline 
Reference only & 84/1/0 & 508/2/0 & 501/9/0 & 488/22/0 & 471/39/0 & 424/86/0 & 303/207/0 & 156/354/0 \\
All data & 0/85/0 & 0/510/0 & 82/428/0 & 158/352/0 & 215/295/0 & 241/269/0 & 223/287/0 & 168/342/0 \\
Median of probs. & 9/76/0 & 30/480/0 & 53/457/0 & 93/417/0 & 189/321/0 & 272/238/0 & 252/258/0 & 216/294/0 \\
\cite{feng2014distributed} & 8/77/0 & 28/482/0 & 19/491/0 & 84/426/0 & 172/338/0 & 254/256/0 & 253/257/0 & 217/293/0 \\
\cite{pmlr-v80-yin18a} & 14/71/0 & 123/387/0 & 227/283/0 & 155/355/0 & 247/259/4 & 295/215/0 & 282/228/0 & 224/286/0 \\
\cite{pregibon1982resistant} & 55/30/0 & 287/223/0 & 282/228/0 & 329/181/0 & 350/160/0 & 358/152/0 & 374/136/0 & 367/143/0 \\
Batch norm & 0/85/0 & 2/508/0 & 78/432/0 & 139/370/1 & 183/326/1 & 186/323/1 & 155/354/1 & 97/412/1 \\
 \hline
 \end{tabular}
\end{table*}

\begin{table*}[h!]
\caption{Summary of results for $p = 0.5$, split by the type of data corruption}
\begin{subtable}{1\textwidth}
\centering
\caption{Summary of the results for $p = 0.5$ and label bias, over all 85 prediction tasks}
 \label{table:animals_p_half_enforce_label}
 \begin{tabular}{|c || c | c | c | c | c | c || c |} 
 \hline
 \backslashbox{Baseline}{$n$} & $n = 10$ & $n = 20$ & $n = 30$ & $n = 40$ & $n = 50$ & $n = 55$ & $n = 59$ \\ 
 \hline\hline 
Reference only & 85/0/0 & 84/1/0 & 82/3/0 & 79/6/0 & 66/19/0 & 39/46/0 & 0/85/0 \\
All data & 0/85/0 & 52/33/0 & 73/12/0 & 82/3/0 & 84/1/0 & 84/1/0 & 84/1/0 \\
Median of probs. & 9/76/0 & 41/44/0 & 72/13/0 & 80/5/0 & 84/1/0 & 84/1/0 & 84/1/0 \\
\cite{feng2014distributed} & 5/80/0 & 7/78/0 & 62/23/0 & 76/9/0 & 83/2/0 & 84/1/0 & 84/1/0 \\
\cite{pmlr-v80-yin18a} & 27/58/0 & 57/28/0 & 48/37/0 & 79/6/0 & 84/1/0 & 83/2/0 & 83/2/0 \\
\cite{pregibon1982resistant} & 48/37/0 & 49/36/0 & 58/27/0 & 65/20/0 & 70/15/0 & 79/6/0 & 84/1/0 \\
Batch norm & 1/84/0 & 46/39/0 & 69/16/0 & 76/9/0 & 79/6/0 & 76/9/0 & 76/9/0 \\
 \hline
 \end{tabular}
 
\end{subtable}
\hfill
\begin{subtable}{1\textwidth}
\centering
 \caption{Summary of the results for $p = 0.5$ and shuffled labels, over all 85 prediction tasks}
 \label{table:animals_p_half_shuffle_labels}
 \begin{tabular}{|c || c | c | c | c | c | c || c |} 
 \hline
 \backslashbox{Baseline}{$n$} & $n = 10$ & $n = 20$ & $n = 30$ & $n = 40$ & $n = 50$ & $n = 55$ & $n = 59$ \\ 
 \hline\hline 
Reference only & 84/1/0 & 82/3/0 & 72/13/0 & 59/26/0 & 42/43/0 & 28/57/0 & 17/68/0 \\
All data & 0/85/0 & 0/85/0 & 9/76/0 & 31/54/0 & 47/38/0 & 50/35/0 & 49/36/0 \\
Median of probs. & 3/82/0 & 0/85/0 & 0/85/0 & 13/72/0 & 40/45/0 & 47/38/0 & 49/36/0 \\
\cite{feng2014distributed} & 4/81/0 & 0/85/0 & 0/85/0 & 4/81/0 & 29/56/0 & 44/41/0 & 48/37/0 \\
\cite{pmlr-v80-yin18a} & 25/60/0 & 44/41/0 & 17/68/0 & 14/67/4 & 19/66/0 & 28/57/0 & 39/46/0 \\
\cite{pregibon1982resistant} & 56/29/0 & 44/41/0 & 59/26/0 & 59/26/0 & 60/25/0 & 67/18/0 & 64/21/0 \\
Batch norm & 0/85/0 & 0/85/0 & 1/83/1 & 10/74/1 & 4/80/1 & 2/82/1 & 2/82/1 \\
 \hline
 \end{tabular}
\end{subtable}
\hfill
\begin{subtable}{1\textwidth}
\centering
 \caption{Summary of the results for $p = 0.5$ and shuffled features, over all 85 prediction tasks}
 \label{table:animals_p_half_shuffle_inputs}
 \begin{tabular}{|c || c | c | c | c | c | c || c |} 
 \hline
 \backslashbox{Baseline}{$n$} & $n = 10$ & $n = 20$ & $n = 30$ & $n = 40$ & $n = 50$ & $n = 55$ & $n = 59$ \\ 
 \hline\hline 
Reference only & 84/1/0 & 85/0/0 & 83/2/0 & 82/3/0 & 76/9/0 & 44/41/0 & 0/85/0 \\
All data & 0/85/0 & 26/59/0 & 51/34/0 & 57/28/0 & 57/28/0 & 41/44/0 & 0/85/0 \\
Median of probs. & 6/79/0 & 4/81/0 & 10/75/0 & 56/29/0 & 69/16/0 & 50/35/0 & 6/79/0 \\
\cite{feng2014distributed} & 6/79/0 & 4/81/0 & 14/71/0 & 61/24/0 & 72/13/0 & 56/29/0 & 6/79/0 \\
\cite{pmlr-v80-yin18a} & 9/76/0 & 18/67/0 & 49/36/0 & 73/12/0 & 77/8/0 & 67/18/0 & 6/79/0 \\
\cite{pregibon1982resistant} & 49/36/0 & 48/37/0 & 54/31/0 & 60/25/0 & 60/25/0 & 56/29/0 & 50/35/0 \\
Batch norm & 1/84/0 & 32/53/0 & 53/32/0 & 60/25/0 & 61/24/0 & 55/30/0 & 14/71/0 \\
 \hline
 \end{tabular}
\end{subtable}
\end{table*}

\pagebreak

\begin{table*}
\ContinuedFloat
\begin{subtable}{1\textwidth}
\centering
 \caption{Summary of the results for $p = 0.5$ and blurred images, over all 85 prediction tasks}
 \label{table:animals_p_half_blured}
 \begin{tabular}{|c || c | c | c | c | c | c || c |} 
 \hline
 \backslashbox{Baseline}{$n$} & $n = 10$ & $n = 20$ & $n = 30$ & $n = 40$ & $n = 50$ & $n = 55$ & $n = 59$ \\ 
 \hline\hline 
Reference only & 85/0/0 & 83/2/0 & 83/2/0 & 83/2/0 & 82/3/0 & 75/10/0 & 67/18/0 \\
All data & 0/85/0 & 0/85/0 & 0/85/0 & 0/85/0 & 0/85/0 & 0/85/0 & 0/85/0 \\
Median of probs. & 4/81/0 & 3/82/0 & 4/81/0 & 6/79/0 & 17/68/0 & 14/71/0 & 24/61/0 \\
\cite{feng2014distributed} & 4/81/0 & 3/82/0 & 3/82/0 & 5/80/0 & 16/69/0 & 16/69/0 & 26/59/0 \\
\cite{pmlr-v80-yin18a} & 19/66/0 & 31/54/0 & 10/75/0 & 29/56/0 & 34/51/0 & 32/53/0 & 38/47/0 \\
\cite{pregibon1982resistant} & 58/27/0 & 58/27/0 & 63/22/0 & 67/18/0 & 70/15/0 & 72/13/0 & 73/12/0 \\
Batch norm & 0/85/0 & 0/85/0 & 0/85/0 & 0/85/0 & 0/85/0 & 0/85/0 & 0/85/0 \\
 \hline
 \end{tabular}
\end{subtable}
\hfill
\begin{subtable}{1\textwidth}
\centering
 \caption{Summary of the results for $p = 0.5$ and dead pixels, over all 85 prediction tasks}
 \label{table:animals_p_half_pixels}
 \begin{tabular}{|c || c | c | c | c | c | c || c |} 
 \hline
 \backslashbox{Baseline}{$n$} & $n = 10$ & $n = 20$ & $n = 30$ & $n = 40$ & $n = 50$ & $n = 55$ & $n = 59$ \\ 
 \hline\hline 
Reference only & 85/0/0 & 83/2/0 & 84/1/0 & 84/1/0 & 80/5/0 & 70/15/0 & 62/23/0 \\
All data & 0/85/0 & 0/85/0 & 0/85/0 & 0/85/0 & 0/85/0 & 0/85/0 & 0/85/0 \\
Median of probs. & 4/81/0 & 3/82/0 & 4/81/0 & 4/81/0 & 10/75/0 & 9/76/0 & 18/67/0 \\
\cite{feng2014distributed} & 4/81/0 & 3/82/0 & 3/82/0 & 4/81/0 & 9/76/0 & 10/75/0 & 16/69/0 \\
\cite{pmlr-v80-yin18a} & 20/65/0 & 32/53/0 & 11/74/0 & 21/64/0 & 29/56/0 & 20/65/0 & 21/64/0 \\
\cite{pregibon1982resistant} & 25/60/0 & 24/61/0 & 37/48/0 & 37/48/0 & 38/47/0 & 40/45/0 & 37/48/0 \\
Batch norm & 0/85/0 & 0/85/0 & 0/85/0 & 2/83/0 & 4/81/0 & 4/81/0 & 1/84/0 \\
 \hline
 \end{tabular}
\end{subtable}
\hfill
\begin{subtable}{1\textwidth}
\centering
 \caption{Summary of the results for $p = 0.5$ and RGB channels swapped, over all 85 prediction tasks}
 \label{table:animals_p_half_RGB}
 \begin{tabular}{|c || c | c | c | c | c | c || c |} 
 \hline
 \backslashbox{Baseline}{$n$} & $n = 10$ & $n = 20$ & $n = 30$ & $n = 40$ & $n = 50$ & $n = 55$ & $n = 59$ \\ 
 \hline\hline 
Reference only & 85/0/0 & 84/1/0 & 84/1/0 & 84/1/0 & 78/7/0 & 47/38/0 & 10/75/0 \\
All data & 0/85/0 & 4/81/0 & 25/60/0 & 45/40/0 & 53/32/0 & 48/37/0 & 35/50/0 \\
Median of probs. & 4/81/0 & 2/83/0 & 3/82/0 & 30/55/0 & 52/33/0 & 48/37/0 & 35/50/0 \\
\cite{feng2014distributed} & 5/80/0 & 2/83/0 & 2/83/0 & 22/63/0 & 45/40/0 & 43/42/0 & 37/48/0 \\
\cite{pmlr-v80-yin18a} & 23/62/0 & 45/40/0 & 20/65/0 & 31/54/0 & 52/33/0 & 52/33/0 & 37/48/0 \\
\cite{pregibon1982resistant} & 51/34/0 & 59/26/0 & 58/27/0 & 62/23/0 & 60/25/0 & 60/25/0 & 59/26/0 \\
Batch norm & 0/85/0 & 0/85/0 & 16/69/0 & 35/50/0 & 38/47/0 & 18/67/0 & 4/81/0 \\
 \hline
 \end{tabular}
\end{subtable}
\end{table*}


\begin{table*}[h!]
\centering
\caption{Summary of the results for $p = 0.2$, over all 85 prediction tasks and all corruptions.}
 \label{table:animals_p_small_appendix}
 \begin{tabular}{|c || c || c | c | c | c | c | c || c |} 
 \hline
 \backslashbox{Baseline}{$n$} & $n = 0$ & $n = 10$ & $n = 20$ & $n = 30$ & $n = 40$ & $n = 50$ & $n = 55$ & $n = 59$ \\ 
 \hline\hline 
Reference only & 84/1/0 & 507/3/0 & 505/5/0 & 504/6/0 & 492/18/0 & 459/51/0 & 429/81/0 & 404/106/0 \\
All data & 0/85/0 & 0/510/0 & 0/510/0 & 0/510/0 & 0/510/0 & 1/509/0 & 2/508/0 & 1/509/0 \\
Median of probs. & 9/76/0 & 28/482/0 & 21/489/0 & 16/494/0 & 17/493/0 & 28/482/0 & 30/478/2 & 31/479/0 \\
\cite{feng2014distributed} & 8/77/0 & 30/480/0 & 24/486/0 & 16/494/0 & 16/494/0 & 20/489/1 & 23/485/2 & 26/484/0 \\
\cite{pmlr-v80-yin18a} & 14/71/0 & 95/415/0 & 146/364/0 & 34/476/0 & 42/468/0 & 39/471/0 & 36/474/0 & 40/470/0 \\
\cite{pregibon1982resistant} & 55/30/0 & 282/228/0 & 282/228/0 & 275/235/0 & 287/223/0 & 264/246/0 & 281/229/0 & 267/243/0 \\
Batch norm & 0/85/0 & 0/510/0 & 0/510/0 & 0/510/0 & 0/510/0 & 0/509/1 & 0/509/1 & 1/508/1 \\
 \hline
 \end{tabular}
\end{table*}

\begin{table*}[h!]
\caption{Summary of results for $p = 0.2$, split by the type of data corruption}
\begin{subtable}{1\textwidth}
\centering
\caption{Summary of the results for $p = 0.2$ and label bias, over all 85 prediction tasks}
 \label{table:animals_p_small_enforce_label}
 \begin{tabular}{|c || c | c | c | c | c | c || c |} 
 \hline
 \backslashbox{Baseline}{$n$} & $n = 10$ & $n = 20$ & $n = 30$ & $n = 40$ & $n = 50$ & $n = 55$ & $n = 59$ \\ 
 \hline\hline 
Reference only & 85/0/0 & 84/1/0 & 84/1/0 & 76/9/0 & 60/25/0 & 46/39/0 & 28/57/0 \\
All data & 0/85/0 & 0/85/0 & 0/85/0 & 0/85/0 & 1/84/0 & 2/83/0 & 1/84/0 \\
Median of probs. & 5/80/0 & 1/84/0 & 0/85/0 & 0/85/0 & 11/74/0 & 13/70/2 & 13/72/0 \\
\cite{feng2014distributed} & 5/80/0 & 4/81/0 & 0/85/0 & 0/85/0 & 3/81/1 & 5/78/2 & 8/77/0 \\
\cite{pmlr-v80-yin18a} & 25/60/0 & 49/36/0 & 7/78/0 & 19/66/0 & 12/73/0 & 4/81/0 & 2/83/0 \\
\cite{pregibon1982resistant} & 47/38/0 & 40/45/0 & 48/37/0 & 44/41/0 & 40/45/0 & 44/41/0 & 41/44/0 \\
Batch norm & 0/85/0 & 0/85/0 & 0/85/0 & 0/85/0 & 0/85/0 & 0/85/0 & 0/85/0 \\
 \hline
 \end{tabular}
 
\end{subtable}
\hfill
\begin{subtable}{1\textwidth}
\centering
 \caption{Summary of the results for $p = 0.2$ and shuffled labels, over all 85 prediction tasks}
 \label{table:animals_p_small_shuffle_labels}
 \begin{tabular}{|c || c | c | c | c | c | c || c |} 
 \hline
 \backslashbox{Baseline}{$n$} & $n = 10$ & $n = 20$ & $n = 30$ & $n = 40$ & $n = 50$ & $n = 55$ & $n = 59$ \\ 
 \hline\hline 
Reference only & 84/1/0 & 84/1/0 & 84/1/0 & 84/1/0 & 77/8/0 & 74/11/0 & 73/12/0 \\
All data & 0/85/0 & 0/85/0 & 0/85/0 & 0/85/0 & 0/85/0 & 0/85/0 & 0/85/0 \\
Median of probs. & 4/81/0 & 4/81/0 & 1/84/0 & 1/84/0 & 1/84/0 & 0/85/0 & 0/85/0 \\
\cite{feng2014distributed} & 5/80/0 & 4/81/0 & 1/84/0 & 1/84/0 & 1/84/0 & 0/85/0 & 0/85/0 \\
\cite{pmlr-v80-yin18a} & 17/68/0 & 31/54/0 & 4/81/0 & 1/84/0 & 1/84/0 & 0/85/0 & 2/83/0 \\
\cite{pregibon1982resistant} & 47/38/0 & 54/31/0 & 49/36/0 & 53/32/0 & 52/33/0 & 54/31/0 & 51/34/0 \\
Batch norm & 0/85/0 & 0/85/0 & 0/85/0 & 0/85/0 & 0/84/1 & 0/84/1 & 0/84/1 \\
 \hline
 \end{tabular}
\end{subtable}
\hfill
\begin{subtable}{1\textwidth}
\centering
 \caption{Summary of the results for $p = 0.2$ and shuffled features, over all 85 prediction tasks}
 \label{table:animals_p_small_shuffle_inputs}
 \begin{tabular}{|c || c | c | c | c | c | c || c |} 
 \hline
 \backslashbox{Baseline}{$n$} & $n = 10$ & $n = 20$ & $n = 30$ & $n = 40$ & $n = 50$ & $n = 55$ & $n = 59$ \\ 
 \hline\hline 
Reference only & 84/1/0 & 85/0/0 & 83/2/0 & 82/3/0 & 73/12/0 & 66/19/0 & 61/24/0 \\
All data & 0/85/0 & 0/85/0 & 0/85/0 & 0/85/0 & 0/85/0 & 0/85/0 & 0/85/0 \\
Median of probs. & 5/80/0 & 4/81/0 & 4/81/0 & 4/81/0 & 3/82/0 & 3/82/0 & 2/83/0 \\
\cite{feng2014distributed} & 5/80/0 & 4/81/0 & 4/81/0 & 4/81/0 & 3/82/0 & 3/82/0 & 2/83/0 \\
\cite{pmlr-v80-yin18a} & 8/77/0 & 5/80/0 & 4/81/0 & 4/81/0 & 3/82/0 & 3/82/0 & 2/83/0 \\
\cite{pregibon1982resistant} & 50/35/0 & 44/41/0 & 44/41/0 & 52/33/0 & 49/36/0 & 51/34/0 & 39/46/0 \\
Batch norm & 0/85/0 & 0/85/0 & 0/85/0 & 0/85/0 & 0/85/0 & 0/85/0 & 0/85/0 \\
 \hline
 \end{tabular}
\end{subtable}
\end{table*}

\pagebreak

\begin{table*}
\ContinuedFloat
\begin{subtable}{1\textwidth}
\centering
 \caption{Summary of the results for $p = 0.2$ and blurred images, over all 85 prediction tasks}
 \label{table:animals_p_small_blured}
 \begin{tabular}{|c || c | c | c | c | c | c || c |} 
 \hline
 \backslashbox{Baseline}{$n$} & $n = 10$ & $n = 20$ & $n = 30$ & $n = 40$ & $n = 50$ & $n = 55$ & $n = 59$ \\ 
 \hline\hline 
Reference only & 85/0/0 & 84/1/0 & 84/1/0 & 83/2/0 & 83/2/0 & 84/1/0 & 84/1/0 \\
All data & 0/85/0 & 0/85/0 & 0/85/0 & 0/85/0 & 0/85/0 & 0/85/0 & 0/85/0 \\
Median of probs. & 4/81/0 & 3/82/0 & 4/81/0 & 6/79/0 & 7/78/0 & 10/75/0 & 10/75/0 \\
\cite{feng2014distributed} & 4/81/0 & 3/82/0 & 4/81/0 & 6/79/0 & 7/78/0 & 10/75/0 & 11/74/0 \\
\cite{pmlr-v80-yin18a} & 15/70/0 & 15/70/0 & 6/79/0 & 9/76/0 & 14/71/0 & 18/67/0 & 22/63/0 \\
\cite{pregibon1982resistant} & 57/28/0 & 56/29/0 & 56/29/0 & 55/30/0 & 55/30/0 & 55/30/0 & 58/27/0 \\
Batch norm & 0/85/0 & 0/85/0 & 0/85/0 & 0/85/0 & 0/85/0 & 0/85/0 & 0/85/0 \\
 \hline
 \end{tabular}
\end{subtable}
\hfill
\begin{subtable}{1\textwidth}
\centering
 \caption{Summary of the results for $p = 0.2$ and dead pixels, over all 85 prediction tasks}
 \label{table:animals_p_small_pixels}
 \begin{tabular}{|c || c | c | c | c | c | c || c |} 
 \hline
 \backslashbox{Baseline}{$n$} & $n = 10$ & $n = 20$ & $n = 30$ & $n = 40$ & $n = 50$ & $n = 55$ & $n = 59$ \\ 
 \hline\hline 
Reference only & 85/0/0 & 84/1/0 & 84/1/0 & 83/2/0 & 84/1/0 & 82/3/0 & 83/2/0 \\
All data & 0/85/0 & 0/85/0 & 0/85/0 & 0/85/0 & 0/85/0 & 0/85/0 & 0/85/0 \\
Median of probs. & 6/79/0 & 5/80/0 & 4/81/0 & 4/81/0 & 4/81/0 & 3/82/0 & 5/80/0 \\
\cite{feng2014distributed} & 6/79/0 & 5/80/0 & 4/81/0 & 3/82/0 & 4/81/0 & 4/81/0 & 4/81/0 \\
\cite{pmlr-v80-yin18a} & 12/73/0 & 14/71/0 & 7/78/0 & 6/79/0 & 6/79/0 & 9/76/0 & 10/75/0 \\
\cite{pregibon1982resistant} & 30/55/0 & 36/49/0 & 26/59/0 & 27/58/0 & 23/62/0 & 25/60/0 & 26/59/0 \\
Batch norm & 0/85/0 & 0/85/0 & 0/85/0 & 0/85/0 & 0/85/0 & 0/85/0 & 1/84/0 \\
 \hline
 \end{tabular}
\end{subtable}
\hfill
\begin{subtable}{1\textwidth}
\centering
 \caption{Summary of the results for $p = 0.2$ and RGB channels swapped, over all 85 prediction tasks}
 \label{table:animals_p_small_RGB}
 \begin{tabular}{|c || c | c | c | c | c | c || c |} 
 \hline
 \backslashbox{Baseline}{$n$} & $n = 10$ & $n = 20$ & $n = 30$ & $n = 40$ & $n = 50$ & $n = 55$ & $n = 59$ \\ 
 \hline\hline 
Reference only & 84/1/0 & 84/1/0 & 85/0/0 & 84/1/0 & 82/3/0 & 77/8/0 & 75/10/0 \\
All data & 0/85/0 & 0/85/0 & 0/85/0 & 0/85/0 & 0/85/0 & 0/85/0 & 0/85/0 \\
Median of probs. & 4/81/0 & 4/81/0 & 3/82/0 & 2/83/0 & 2/83/0 & 1/84/0 & 1/84/0 \\
\cite{feng2014distributed} & 5/80/0 & 4/81/0 & 3/82/0 & 2/83/0 & 2/83/0 & 1/84/0 & 1/84/0 \\
\cite{pmlr-v80-yin18a} & 18/67/0 & 32/53/0 & 6/79/0 & 3/82/0 & 3/82/0 & 2/83/0 & 2/83/0 \\
\cite{pregibon1982resistant} & 51/34/0 & 52/33/0 & 52/33/0 & 56/29/0 & 45/40/0 & 52/33/0 & 52/33/0 \\
Batch norm & 0/85/0 & 0/85/0 & 0/85/0 & 0/85/0 & 0/85/0 & 0/85/0 & 0/85/0 \\

 \hline
 \end{tabular}
\end{subtable}
\end{table*}

\end{document}